\documentclass{article}
\pdfoutput=1 

\usepackage{microtype}
\usepackage{graphicx}
\usepackage{subfigure}
\usepackage{booktabs} % for professional tables

\usepackage{multirow}
\usepackage{hyperref}

\usepackage[accepted]{icml2018}

\icmltitlerunning{A probabilistic framework for multi-view feature learning with many-to-many associations via neural networks}

\usepackage{amsmath,amssymb}
\usepackage{algorithmic,algorithm}
\usepackage{subfigure}
\usepackage{multirow}
\usepackage{comment}
\usepackage{enumerate}
\usepackage{wrapfig}
\usepackage{ascmac}

\usepackage{mdframed}

\usepackage{titling}
\usepackage{natbib}

\usepackage{framed}

\newcommand{\bs}{\boldsymbol}
\newcommand{\x}{\bs x}

\newcommand{\y}{\bs y}

\newcommand{\W}{\bs W}

\newcommand{\tr}{\mathrm{tr}}

\newcommand{\ave}{E}

\newcommand{\argmax}{\mathop{\arg\max}}
\newcommand{\arginf}{\mathop{\arg\inf}}
\newcommand{\argsup}{\mathop{\arg\sup}}

\def\qed{\hfill $\Box$} 

\newtheorem{proof}{Proof}[section]
\allowdisplaybreaks[1]

\begin{document}

\twocolumn[
%\icmltitle{A probabilistic framework for multi-view feature learning \\ with neural networks}
\icmltitle{A probabilistic framework for multi-view feature learning \\ with many-to-many associations via neural networks}

% It is OKAY to include author information, even for blind
% submissions: the style file will automatically remove it for you
% unless you've provided the [accepted] option to the icml2018
% package.

% List of affiliations: The first argument should be a (short)
% identifier you will use later to specify author affiliations
% Academic affiliations should list Department, University, City, Region, Country
% Industry affiliations should list Company, City, Region, Country

% You can specify symbols, otherwise they are numbered in order.
% Ideally, you should not use this facility. Affiliations will be numbered
% in order of appearance and this is the preferred way.
\icmlsetsymbol{equal}{*}

\begin{icmlauthorlist}
\icmlauthor{Akifumi Okuno}{kyoto,riken}
\icmlauthor{Tetsuya Hada}{recruit}
\icmlauthor{Hidetoshi Shimodaira}{kyoto,riken}
\end{icmlauthorlist}

\icmlaffiliation{kyoto}{Graduate School of Informatics, Kyoto University, Kyoto, Japan}
\icmlaffiliation{riken}{RIKEN Center for Advanced Intelligence Project~(AIP), Tokyo, Japan}
\icmlaffiliation{recruit}{Recruit Technologies Co., Ltd., Tokyo, Japan}

\icmlcorrespondingauthor{Akifumi Okuno}{okuno@sys.i.kyoto-u.ac.jp}

% You may provide any keywords that you
% find helpful for describing your paper; these are used to populate
% the "keywords" metadata in the PDF but will not be shown in the document
\icmlkeywords{Multi-view, Graph embedding, Neural networks}

\vskip 0.3in
]

% this must go after the closing bracket ] following \twocolumn[ ...

% This command actually creates the footnote in the first column
% listing the affiliations and the copyright notice.
% The command takes one argument, which is text to display at the start of the footnote.
% The \icmlEqualContribution command is standard text for equal contribution.
% Remove it (just {}) if you do not need this facility.

\newtheorem{theo}{Theorem}[section]
\newtheorem{defi}{Definition}[section]
\newtheorem{lemm}{Lemma}[section]
\newtheorem{prop}{Proposition}[section]

\printAffiliationsAndNotice{}  % leave blank if no need to mention equal contribution
%\printAffiliationsAndNotice{\icmlEqualContribution} % otherwise use the standard text.

\begin{abstract}
A simple framework Probabilistic Multi-view Graph Embedding~(PMvGE) is proposed for multi-view feature learning with many-to-many associations so that it generalizes various existing multi-view methods. 
PMvGE is a probabilistic model for predicting new associations via graph embedding of the nodes of data vectors with links of their associations.
Multi-view data vectors with many-to-many associations are transformed by neural networks to feature vectors in a shared space, and the probability of new association between two data vectors is modeled by the inner product of their feature vectors. 
While existing multi-view feature learning techniques can treat only either of many-to-many association or non-linear transformation, PMvGE can treat both simultaneously.  
By combining Mercer's theorem and the universal approximation theorem, we prove that PMvGE learns a wide class of similarity measures across views.
Our likelihood-based estimator enables efficient computation of non-linear transformations of data vectors in large-scale datasets by minibatch SGD, and numerical experiments illustrate that PMvGE outperforms existing multi-view methods.
\end{abstract}

\begin{table}[htbp]
\centering
\caption{Comparison of PMvGE with existing multi-view / graph-embedding methods.
(Nv): Number of views,
(MM): Many-to-many, (NL): Non-linear, (Ind): Inductive, (Lik): Likelihood-based. PMvGE has all the properties. $\text{Nv}=D$ represents that the method can deal with arbitrary number of views.}
\label{table:comparison}
\scalebox{0.9}{
\begin{tabular}{l|ccccc}
\hline
 & (Nv) & (MM) & (NL) & (Ind) & (Lik) \\
\hline
CCA & 2 &  & &  \checkmark & \\
DCCA & 2 & & \checkmark & \checkmark &  \\
MCCA & $D$ &  & &  \checkmark & \\
SGE & 0 & \checkmark & &  & \\
LINE & 0 & \checkmark & &  & \checkmark\\
LPP & 1 & \checkmark & & \checkmark & \\
CvGE & 2 & \checkmark & &  \checkmark &\\
CDMCA & $D$ & \checkmark & &  \checkmark &\\
DeepWalk & 0 & \checkmark  &  & & \checkmark \\
SBM & 1 & \checkmark &  & \checkmark & \checkmark \\
GCN & 1 & \checkmark & \checkmark &  &\checkmark\\
GraphSAGE & 1 & \checkmark & \checkmark & \checkmark & \checkmark \\
IDW & 1 & \checkmark & \checkmark & \checkmark & \checkmark \\
%\hline
\textbf{PMvGE} & $D$ & \checkmark & \checkmark & \checkmark & \checkmark\\
\hline
\end{tabular}
}
\vspace{-2em}
\end{table}

\section{Introduction}

%Multi-viewのデータ解析の重要性
With the rapid development of Internet communication tools in past few decades, many different types of data vectors become easily obtainable these days,
advancing the development of multi-view data analysis methods~\citep{sun2013survey,zhao2017multi}.
Different types of vectors are called as ``views", and their dimensions may be different depending on the view. 
Typical examples are data vectors of images, text tags, and user attributes available in Social Networking Services~(SNS).
However, we cannot apply standard data analysis methods, such as clustering, to multi-view data vectors, because data vectors from different views, say, images and text tags, are not directly compared with each other.
In this paper, we work on multi-view Feature Learning for transforming the data vectors from all the views into new vector representations called ``feature vectors" in a shared euclidean subspace.

One of the best known approaches to multi-view
feature learning is Canonical Correlation Analysis~\citep[CCA]{hotelling1936relations} for two-views, and Multiset CCA~\citep[MCCA]{kettenring1971canonical} for many views. 
CCA considers pairs of related data vectors $\{(\bs x^{(1)}_i,\bs x^{(2)}_i)\}_{i=1}^{n} \subset \mathbb{R}^{p_1} \times \mathbb{R}^{p_2}$. For instance, $\bs x^{(1)}_i \in \mathbb{R}^{p_1}$ may represent an image, and $\bs x^{(2)}_i \in \mathbb{R}^{p_2}$ may represent a text tag. 
Their dimension $p_1$ and $p_2$ may be different. 
CCA finds linear transformation matrices $\bs A^{(1)},\bs A^{(2)}$ so that the sum of inner products $\sum_{i=1}^{n} \langle \bs A^{(1)\top}\bs x^{(1)}_i,\bs A^{(2)\top}\bs x^2_i \rangle$ is maximized under a variance constraint. 
The obtained linear transformations compute feature vectors $\bs y^{(1)}_i:=\bs A^{(1)\top}\bs x^{(1)}_i,\bs y^{(2)}_i:=\bs A^{(2)\top}\bs x^{(2)}_i \in \mathbb{R}^K$ where $K \leq \min\{p_1,p_2\}$ is the dimension of the shared space of feature vectors from the two views. However, the linear transformations may not capture the underlying structure of real-world datasets due to its simplicity.

To enhance the expressiveness of transformations, CCA has been further extended to non-linear settings, as Kernel CCA~\citep{lai2000kernel} and Deep CCA~\citep[DCCA]{andrew2013deep,wang2016deep} which incorporate kernel methods and neural networks to CCA, respectively. 
These methods show drastic improvements in performance in 
face recognition~\citep{zheng2006facial} 
%, pose estimation~\citep{li2015maximum}, 
and image-text embedding~\citep{yan2015deep}. However, these CCA-based approaches are limited to multi-view data vectors with one-to-one correspondence across views.

%% CCAは1-to-1を仮定してるけど実際にはmany-to-manyの場合がある
%While CCA-based approaches are limited to one-to-one correspondence, 
Real-world datasets often have more complex association structures among the data vectors, thus
the whole dataset is interpreted as a large graph with nodes of data vectors and links of the associations.
For example, associations between images $\{\bs x^{(1)}_i\}_{i=1}^{n_1}$ and their tags $\{\bs x^{(2)}_j \}_{j=1}^{n_2}$ may be many-to-many relationships, in the sense that each image has multiple associated tags as well as each tag has multiple associated images. 
The weight $w_{ij} \geq 0$, which we call ``link weight", is defined to represent the strength of association between data vectors $\bs x^{(1)}_i,\bs x^{(2)}_j$ ($i=1,2,\ldots,n_1;j=1,2,\ldots,n_2$). The number of data vectors $n_1,n_2$ in each view may be different.

%% many-to-many associationを扱うためにCvGEとかが提案された
To fully utilize the complex associations represented by $\{w_{ij}\}$, Cross-view Graph Embedding~\citep[CvGE]{huang2012cross} and its extension to more than three views called Cross-Domain Matching Correlation Analysis~\citep[CDMCA]{shimodaira2016cross} are proposed recently, by extending CCA to many-to-many settings. 
2-view CDMCA ($=\,$CvGE) obtains linear transformation matrices $\bs A^{(1)},\bs A^{(2)}$ so that the sum of inner products $\sum_{i=1}^{n_1}\sum_{j=1}^{n_2} w_{ij} \langle \bs A^{(1)\top}\bs x^{(1)}_i,\bs A^{(2)\top}\bs x^{(2)}_j \rangle$ is maximized under a variance constraint.

CDMCA includes various existing multi-view / graph-embedding methods as special cases. For instance, 2-view CDMCA obviously includes CCA as a special case $w_{ij}=\delta_{ij}$, where $\delta_{ij}$ is Kronecker's delta. 
By considering $1$-view setting, where $\bs x_i \in \mathbb{R}^{p}$ is a 1-view data vector and $w_{ij} \geq 0$ is a link weight between $\x_i$ and $\x_j$, CDMCA reduces to Locality Preserving Projections~\citep[LPP]{he2003locality,yan2007graph}. 
LPP also reduces to Spectral Graph Embedding~\citep[SGE]{chung1997spectral,belkin2001laplacian}, which is interpreted as ``0-view'' graph embedding, by letting $\bs x_i \in \{0,1\}^n$ be 1-hot vector with 1 at $i$-th entry and 0 otherwise. 

%CDMCA reduces to LPP~\citep{he2003locality} if $D=1$, and LPP reduces to SGE~\citep{belkin2001laplacian} if $\x_i\in \{0,1\}^n$ is 1-hot vector.
%CDMCA reduces to CCA~\citep{hotelling1936relations} if $D=2$, $d_i=1\, (i=1,2, \ldots,n)$, $d_i= 2\, (i=n+1,n+2,\ldots,2n)$, and $w_{ij}=\bs 1(|i-j|=n)$. 
%If the linear transformation of CCA is replaced with the neural networks defined in eq.~(\ref{eq:neural_network}), it becomes Deep CCA~\citep[DCCA]{andrew2013deep}. 
%Regarding the likelihood-based approach, word2vec~\cite{mikolov2013distributed} gives vector representations to words so that the inner product of word vectors infers their binary-occurrence, and GloVe~\citep{pennington2014glove} the number of co-occurrence as well. 

Although CDMCA shows a good performance in word and image embeddings~\citep{oshikiri2016cross,fukui2016image}, its expressiveness is still limited due to its linearity. 
There has been a necessity of a framework, which can deal with many-to-many associations and non-linear transformations simultaneously. 
Therefore, in this paper, we propose a non-linear framework for multi-view feature learning with many-to-many associations. 
We name the framework as Probabilistic Multi-view Graph Embedding~(PMvGE). 
Since PMvGE generalizes CDMCA to non-linear setting, PMvGE can be regarded as a generalization of various existing multi-view methods as well.

PMvGE is built on a simple observation: 
many existing approaches to feature learning consider the inner product similarity of feature vectors. 
For instance,  the objective function of CDMCA is the weighted sum of the inner product $\langle \bs y_i, \bs y_j \rangle$ of the two feature vectors $\bs y_i$ and $\bs y_j$ in the shared space. 
Turning our eyes to recent $1$-view feature learning, Graph Convolutional Network~\citep{kipf2017semi}, GraphSAGE~\citep{hamilton2017inductive}, and Inductive DeepWalk~\citep{dai2018adversarial} assume that the inner product of feature vectors $\langle \bs y_i,\bs y_j \rangle$ approximates link weight $w_{ij} \geq 0$.

Inspired by these existing studies, for $D$-view feature learning ($D\ge1$), PMvGE transforms data vectors $\bs x_i \in \mathbb{R}^{p_{d_i}}$ from view $d_i \in \{1,\ldots, D\}$ by neural networks $\bs x_i \mapsto \bs y_i:=\text{NN}^{(d_i)}(\bs x_i) \in \mathbb{R}^{K} \: (i=1,2,\ldots,n)$ so that the function $\exp(\langle \bs y_i,\bs y_j \rangle)$ approximates the weight $w_{ij}$ for $i,j=1,\ldots,n$. 
%%我々の提案手法は極めてシンプルである:
%view $d_i,d_j$に属するデータ$v_i,v_j$のデータベクトル$\bs x_i \in \mathbb{R}^{p_{d_i}},\bs x_j \in \mathbb{R}^{p_{d_j}}$をそれぞれNNによって変換し，同様に内積で類似度を表現するモデル$w_{ij} \sim \text{Po}(\alpha \exp(\langle \bs y_i,\bs y_j \rangle))$ where $\bs y_i:=\text{NN}^{(d_i)}(\bs x_i)$を提案し，
We introduce a parametric model of the conditional probability of  $\{w_{ij}\}$ given $\{ (\bs x_i, d_i) \}$, and thus PMvGE is a non-linear and probabilistic extension of CDMCA.
This leads to very efficient computation of the Maximum Likelihood Estimator (MLE) with minibatch SGD.

%\subsection{Our contribution}
%% 貢献
Our contribution in this paper is summarized as follows: 

%我々の貢献は以下の通りである．
%\begin{enumerate}[{(1)}]
%\item 
(1) We propose PMvGE for multi-view feature learning with many-to-many associations, which is non-linear, efficient to compute, and inductive. Comparison with existing methods is shown in Table~\ref{table:comparison}.
See Section~\ref{sec:related_works} for the description of these methods.

(2) We show in Section~\ref{sec:proposal} that PMvGE generalizes various existing multi-view methods, at least approximately, by considering the Maximum Likelihood Estimator~(MLE) with a novel probabilistic model.
%Our model can capture a complex real-world data structure by incorporating neural networks. 
%(1) 我々は，提案する確率モデルに基づきMLEを考えることで，既存研究をapproximately generalizeするframeworkを与えた in Section~\ref{subsec:proposed_model}．我々のモデルは線形変換でなく，Neural Networkによる非線形変換を利用することで実データの複雑なデータ構造をcaptureすることができる．

%\item %(2) 
(3) We show in Section~\ref{sec:optimization} that PMvGE with large-scale datasets can be efficiently computed by minibatch SGD.
%(2) 我々のframeworkでは，minibatch SGDを利用することで巨大なデータに対してもMLEが効率的に計算することができることを示した in Section~\ref{subsec:scalable_learning_general}．

(4) We prove that 
%the inner-product of neural networks can approximate a variety of similarity measures 
PMvGE, yet very simple, learns a wide class of similarity measures across views.
By combining Mercer's theorem and the universal approximation theorem, we prove in Section~\ref{subsec:case_neural_network} that
the inner product of feature vectors can approximate arbitrary continuous positive-definite similarity measures
via sufficiently large neural networks.
%This result corresponds to Mercer's theorem if $D=1$, however, our result can also be applicable to multi-view settings.
We also prove in Section~\ref{subsec:consistency_pmvge} that
MLE will actually learn the correct parameter value for sufficiently large number of data vectors.

\section{Related works}

\label{sec:related_works}
%% SBM界隈の研究
%==コンフリクト時の下平先生
%\textbf{0-view FL:} 
%Stochastic block model~\citep[SBM]{holland1983stochastic,nowicki2001estimation} is a well-known probabilistic model of graphs, whose links are generated with probabilities depending on the cluster memberships of nodes. 
%Its variants~\citep{airoldi2008mixed,bickel2009nonparametric}, which automatically detect the mixture of cluster memberships, are often used these days. 
%SBM assumes the cluster structure of nodes, while PMvGE does not. 

%\textbf{1-view FL:} 
%LPP, which stands for Locality Preserving Projections~\citep{he2003locality,yan2007graph}, incorporates a linear transformation of data vectors into Spectral graph Embedding (SE). 

\textbf{0-view feature learning:} 
There are several graph embedding methods related to PMvGE without data vectors. 
We call them as $0$-view feature learning methods. 
%There are several graph embedding methods which do not utilize data vectors~\citep{DBLP:journals/corr/abs-1709-07604}. 
%We call them as $0$-view feature learning methods. 
Given a graph, Spectral Graph Embedding~\citep[SGE]{chung1997spectral,belkin2001laplacian} obtains feature vectors of nodes by considering the adjacency matrix. However, SGE requires time-consuming eigenvector computation due to its variance constraint. 
%SGE imposes a variance constraint, and requires time consuming eigenvector computation.
LINE~\citep{tang2015line} is very similar to 0-view PMvGE, which
reduces the time complexity of SGE by proposing a probabilistic model so that any constraint is not required.
DeepWalk~\citep{perozzi2014deepwalk} and node2vec~\citep{grover2016node2vec} obtain feature vectors of nodes by applying skip-gram model~\citep{mikolov2013distributed} to the node-series computed by random-walk over the given graph, 
while PMvGE directly considers the likelihood function.
%node2vec~\citep{grover2016node2vec} modifies DeepWalk so that the random-walk is randomly re-started some time to obtain better node-series. 
%These methods utilize skip-gram model of node-series, while PMvGE directly considers the likelihood function. 

\textbf{1-view feature learning:} 
Locality Preserving Projections~\citep[LPP]{he2003locality,yan2007graph} incorporates linear transformation of data vectors into SGE. 
%While LPP is limited to linear setting, approaches based on neural networks are attracting much attention~\citep{bronstein2017geometric}. 
For introducing non-linear transformation,
the graph neural network~\citep[GNN]{scarselli2009graph} defines a graph-based convolutional neural network.
ChebNet~\citep{defferrard2016convolutional} and Graph Convolutional Network~(GCN)~\citep{kipf2017semi} reduce the time complexity of GNN by approximating the convolution. 
These GNN-based approaches are highly expressive but not \textit{inductive}. 
A method is called \textit{inductive} if it computes feature vectors for newly obtained vectors which are not included in the training set. 
%GraphSAGE~\cite{hamilton2017inductive} fixes this problem so that it can be applied to unseen data vectors without re-learning.
GraphSAGE~\citep{hamilton2017inductive} and Inductive DeepWalk~\citep[IDW]{dai2018adversarial} are inductive as well as our proposal PMvGE, but probabilistic models of these methods are different. 
%Although they are inductive as well as our proposal PMvGE, their probabilistic models are different. 

%Inductive DeepWalk~\citep[IDW]{dai2018adversarial} is proposed very recently by incorporating data vectors to DeepWalk. 
%Although all of GraphSAGE, IDW, and our proposal PMvGE are inductive, their probabilistic models are different. 
%
%%これらのモデルは高い表現力を持つ一方で，変換$\bs y_i=\text{CNN}_i(\bs X)$にグラフのラプラシアンを利用するので，学習に利用しなかったノードの特徴量が計算できない．したがって，学習に利用しなかったデータと学習済みデータとの類似度を予測することができない．
%%そこで，グラフのラプラシアンを利用する代わりに，近傍ノードのaggregationを利用するモデルが提案された~\citep{hamilton2017inductive}. aggregationのモデルのように，再学習を行うことなくノードの特徴量ベクトルが得られる手法をInductiveであるという~\citep{zhu2006semi,pan2010survey}．
%%\citep{hamilton2017inductive}に加えて，
%%DeepWalk~\citep{perozzi2014deepwalk}のInductive variant~\citep{dai2018adversarial}が提案されている．

\textbf{Multi-view feature learning:} 
%\cite{iwata2016probabilistic} proposes a probabilistic model for multi-view learning with many-to-many associations. This model is based on the mixtures of node-clusters in the latent space, while our model does not assume the cluster structure. 
%CDMCA~\citep{shimodaira2016cross} utilizes a linear transformations of multi-view feature vectors. 
%HIMFAC~\citep{nori2012multinomial} is mathematically equivalent to CDMCA, and they reduce to Cross-view Graph Embedding~(CvGE)~\citep{huang2012cross} if the number of views is $2$. 
%CvGE reduces to CCA~\citep{hotelling1936relations} if the association is one-to-one: $w_{ij}=\delta_{ij}$ where $\delta_{ij}$ is a kronecker's delta. 
%CCA is extended to non-linear setting by incorporating Deep neural networks, as Deep CCA~\citep{andrew2013deep}. 
HIMFAC~\citep{nori2012multinomial} is mathematically equivalent to CDMCA. 
%CDMCA and HIMFAC with more than three-view data vectors reduce to Multiset CCA~\citep[MCCA]{kettenring1971canonical} if one-to-one associations are considered. 
%MCCA reduces to CCA if the number of view is $2$. 
Factorization Machine~\citep[FM]{rendle2010factorization} incorporates higher-order products of multi-view data vectors to linear-regression. It can be used for link prediction across views. 
If only the second terms in FM are considered, FM is approximately the same as PMvGE with linear transformations. 
However, FM does not include PMvGE with neural networks.

\textbf{Another study} 
Stochastic Block Model~\citep[SBM]{holland1983stochastic,nowicki2001estimation} is a well-known probabilistic model of graphs, whose links are generated with probabilities depending on the cluster memberships of nodes. 
%Its variants~\citep{airoldi2008mixed,bickel2009nonparametric}, which automatically detect the mixture of cluster memberships, are often used these days. 
SBM assumes the cluster structure of nodes, while our model does not.

\section{Proposed model and its parameter estimation}
\label{sec:proposal}
%In Section~\ref{subsec:setup}, we describe the problem set-up to propose our novel model and its framework named PMvGE in Section~\ref{subsec:proposed_model}. 
%A specific setting which is required in practice is shown in Section~\ref{subsec:view_missing}, the MLE of our model is defined in Section~\ref{subsec:mle}, and the relationship to existing studies is shown in Section~\ref{subsec:relation}. 

\subsection{Preliminaries}
\label{subsec:setup}
We consider an undirected graph consisting of $n$ nodes $\{v_i\}_{i=1}^{n}$ and link weights 
$w_{ij} \geq 0 \quad (i,j=1,2,\ldots,n)$ satisfying $w_{ij}=w_{ji}$ for all $i,j$, and $w_{ii}=0$.
Let $D \in \mathbb{N}$ be the number of views.
For $D$-view feature learning, node $v_i$ belongs to one of views, which we denote as $d_i \in \{1,2,\ldots,D\}$. The data vector representing the attributes (or side-information) at node $v_i$ is denoted as
$
\x_i \in \mathbb{R}^{p_{d_i}}
$
for view $d_i$ with dimension $p_{d_i}$.
For 0-view feature learning, we formally let $D=1$ and use the 1-hot vector $\bs x_i \in \{0,1\}^n$.
We assume that we obtain $\{w_{ij}\}_{i,j=1}^{n},\{\x_i,d_i\}_{i=1}^{n}$ as observations. 
By taking $w_{ij}$ as a random variable, we consider a parametric model of conditional probability of $w_{ij}$ given the data vectors.
In Section~\ref{subsec:probabilistic_model}, we consider the probability model of $w_{ij}$ with the conditional expected value
\[
	\mu_{ij} = E(w_{ij} | \{\x_i, d_i\}_{i=1}^n )
\]
where $\{\bs x_i,d_i\}_{i=1}^{n}$ is given, for all $1 \le i < j \le n$.
In Section~\ref{subsec:proposed_model}, we then define PMvGE by specifying the functional form of $\mu_{ij}$ via feature vectors.

%Nodes $\{v_i\}_{i=1}^{n}$, weights $\{w_{ij}\}_{i,j=1}^{n}$が得られていて，各node $v_i$が$D$-viewのどれかに属しているとする．
%$d_i \in \{1,2,\ldots,D\}$ represents the view which the node $v_i$ belongs to, and we denote its attribute as $\x_i \in \mathbb{R}^{p_{d_i}}$, whose dimension $p_{d_i}$ depends on the view. 
%我々の問題設定は，ノード$i,j$が属するview-$d_i,d_j$のattribute $\x_i,\x_j$と$w_{ij}$をbridgeする確率モデルを提案することを目的とする．

\subsection{Probabilistic model}
\label{subsec:probabilistic_model}

%\subsection{Preliminaries}
%\label{subsec:preliminary}
%In this section, we consider time development of an undirected graph consisting of fixed $n$ nodes $\{v_i\}_{i=1}^{n}$, to derive a generative model of graphs shown in eq.~(\ref{eq:w_model}). 
%% == このsectionのモデルはdiscriminativeではなくgenerative (?)
%% => ｗだけに注目するとgenerativeですが，ｘあたえたときのwの条件付き確率といういみでdiscriminativeと思います．

%$n$個のノード$\{v_i\}_{i=1}^{n}$に，経時的にリンクが生成されるモデルを考える．
For deriving our probabilistic model, we first consider a random graph model with fixed $n$ nodes.
At each time-point $t=1,2,\ldots,T$, an unordered node pair $(v_i,v_j)$ is chosen randomly with probability
%最終時刻$T$は十分大きいとし，
%時刻$t$で生成されるリンク$e_t$がノード$(v_i,v_j)$間である確率を
$$
\mathbb{P}(e_t=(v_i,v_j))=\frac{\mu_{i'j'}}{\sum_{1 \leq i < j \leq n} \mu_{ij}} 
$$
where $i':=\min\{i,j\},j':=\max\{i,j\}$ and $e_t$ represents the undirected link at time $t$.
The parameters $\mu_{ij} \geq 0 \quad (1 \leq i < j \leq n)$ are interpreted as unnormalized probabilities of node pairs.
We allow the same pair is sampled several times.
Given independent observations $e_1, e_2, \ldots, e_T$,
we consider the number of links generated between $v_i$ and $v_j$ as
%$1-\sum_{i,j=1}^{n} \mathbb{P}(e_t=(i,j))$ 
$$
w_{ij}=w_{ji}:= \sharp \{t \in \{1,\ldots, T\} \mid e_t=(v_i,v_j)\}.
$$
%とする．確率が$T^{-1}$に比例するのは，最終時刻でのリンク数の期待値を$O(1)$にするためであり，
%Note that $\mathbb{P}(\text{時刻$t$ではリンクを張らない})=1-\sum_{i,j=1}^{n} \mathbb{P}(e_t=(i,j))$ can be positive．
The conditional probability $\mathbb{P}(\{w_{ij}\}_{1 \leq i < j \leq n} \mid T)$ follows a multinomial distribution. 
Assuming that $T$ obeys Poisson distribution with mean $\sum_{1 \leq i < j \leq n} \mu_{ij}$, the probability of $\{w_{ij}\}_{1 \leq i < j \leq n}$ follows
$$
\mathbb{P}(\{w_{ij}\}_{1 \leq i < j \leq n})
=
\prod_{1 \leq i < j \leq n} p(w_{ij} ; \mu_{ij})
$$
where $p(w;\mu)$ is the probability function of Poisson distribution with mean $\mu$.
Thus $w_{ij}$ follows Poisson distribution independently as
\begin{align}
w_{ij} \overset{\text{indep.}}{\sim}  \text{Po}(\mu_{ij})
\label{eq:w_model}
\end{align}
for all $1 \leq i < j \leq n$. 
Although $w_{ij}$ should be a non-negative integer as an outcome of Poisson distribution, our likelihood computation allows $w_{ij}$ to take any nonnegative real value.

Our probabilistic model (\ref{eq:w_model}) is nothing but Stochastic Block Model~\citep[SBM]{holland1983stochastic}
by assuming that node $v_i$ belongs to a cluster $c_i \in \{1,2,\ldots,C\}$.
The model is specified as
%に従う．グラフにクラスタ構造を仮定し，特に，ノード$v_i$の属するクラスタを$c_i$としたとき
\begin{align}
\mu_{ij}=\beta^{(c_i,c_j)},\label{eq:sbm_mu_ij}
%\label{eq:muij_sbm}
\end{align}
where $\beta^{(c_i,c_j)}$ is the parameter regulating the number of links whose end-points belong to clusters $c_i$ and $c_j$.
Note that SBM is interpreted as 1-view method with
1-hot vector $\bs x_i \in \{0,1\}^C$ indicating the cluster membership.

\subsection{Proposed model (PMvGE)}
\label{subsec:proposed_model}
Inspired by various existing methods for $0$-view and $1$-view feature learning, we propose a novel model for the parameter $\mu_{ij}$ in eq.~(\ref{eq:w_model}) by using the inner-product similarity as 
\begin{align}
\mu_{ij}(\bs \alpha,\bs \psi)
&:=
\alpha^{(d_i,d_j)}
\exp\left(
	\big \langle 
		\bs y_i,
		\bs y_j
	\big \rangle
\right), \label{eq:mu_ij} \\
\bs y_i&:=f^{(d_i)}_{\bs \psi}(\x_i). \nonumber
\end{align}
Here $\bs \alpha = (\alpha^{(d,e)}) \in \mathbb{R}^{D \times D}_{\geq 0}$ is a symmetric parameter matrix
($\bs \alpha=\bs \alpha^{\top}$) for regulating the sparseness of $\W=(w_{ij})$.
For $\bs y, \bs y' \in \mathbb{R}^K$, $\langle \bs y, \bs y' \rangle = \sum_{i=1}^K y_i y'_i$ is simply the inner product in Euclidean space. 
The functions $f^{(d)}_{\bs \psi}:\mathbb{R}^{p_d} \to \mathbb{R}^{K}$, $d=1,2,\ldots,D$, specify the non-linear transformations from data vectors to feature vectors.
$\bs \psi$ represents a collection of parameters (e.g., neural network weights). 

\begin{figure}[htbp]
\centering
	\includegraphics[width=6cm]{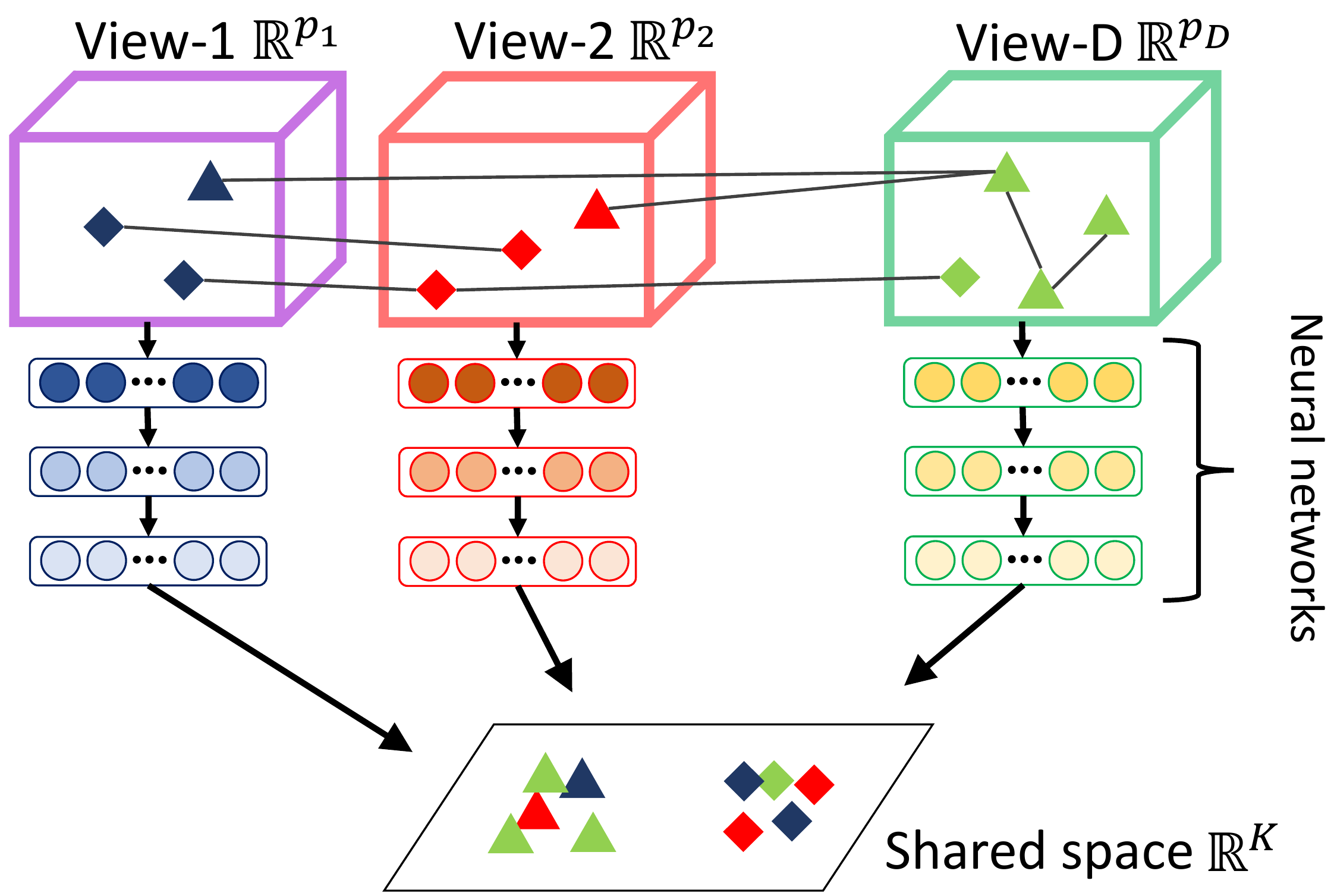}
	\caption{Data vectors $\{\bs x_i\}$ in each view are transformed to feature vectors $\{\bs y_i\}$ by neural networks $\{f^{(d)}_{\bs \psi}\}$.}
	\vspace{-1.5em}
\end{figure}
%transformation of data vectors in view-$よｙd$ to $K$-dimensional common space $(d=1,2,\ldots,D)$. 
These transformations $\{f^{(d)}_{\bs \psi}\}$ can be trained by maximizing the likelihood for the probabilistic model (\ref{eq:w_model}) as shown in Section~\ref{subsec:mle}.
PMvGE computes feature vectors $\{\bs y_i\}_{i=1}^{n}$ through maximum likelihood estimation of transformations $\{f^{(d)}_{\bs \psi}\}$.

%We name the framework, which computes feature vectors $\{\bs y_i\}_{i=1}^{n}$ through maximum likelihood estimation of transformations $f^{(d)}_{\bs \psi} \: (d=1,2,\ldots,D)$, as Probabilistic Multi-view Graph Embedding~(PMvGE). 
%In this paper, we especially consider neural networks defined lateruked as the transformations $\{f^{(d)}_{\bs \psi}\}$. 

%% so that the similarity of data vectors across views can be measured by the inner product similarity $\langle \cdot,\cdot \rangle$. 
%In model~(\ref{eq:mu_ij}), there are two parameter matrices $\bs \alpha,\bs \psi$ to be estimated. 
%$\bs \alpha = (\alpha^{(d,e)}) \in \mathbb{R}^{D \times D}_{\geq 0}$ is a parameter matrix with the constraint $\bs \alpha=\bs \alpha^{\top}$ for regulating the sparseness of $\W$.
%Estimating the suitable parameter $\bs \psi$ is expected to reveal the relationship between data vectors and link weights, which allows us to predict the probability that there should be any link even for new data vectors.
%%Note that the dimension of $\x_i$, which is $p_{d_i}$, depends on which view the corresponding node $v_i$ belongs to. 

PMvGE associates nodes $v_i$ and $v_j$ with the probability specified by the similarity between their feature vectors $\y_i:=f^{(d_i)}_{\bs \psi}(\x_i),\y_j:=f^{(d_j)}_{\bs \psi}(\x_j) \in \mathbb{R}^{K}$ in the shared space.
Nodes with similar feature vectors will share many links. 

% つまりデータの変換が類似しているほど高い確率でリンクが張られる．

We consider the following neural network model for the transformation function, while any functional form can be accepted for PMvGE.
Only the inner product of feature vectors is considered for measuring the similarity in PMvGE.
We prove in Theorem~\ref{theo:universal_approximate} that the inner product with neural networks approximates a wide class of similarity measures.

%\begin{enumerate}[{(1)}]
%\item 
\textbf{Neural Network~(NN)} with 3-layers is defined as 
\begin{align}
f^{(d)}_{\bs \psi}(\x)
=
\bs \sigma(
\bs \psi^{(d)\top}_{3}
\bs \sigma(\bs \psi^{(d)\top}_2
	\bs \sigma(\bs \psi^{(d)\top}_1 \x)
))), 
\label{eq:neural_network}
\end{align}
where $\x \in \mathbb{R}^{p_d}$ is data vector, $\bs \psi^{(d)}_1 \in \mathbb{R}^{p_d \times K^{(d)}_1},\bs \psi^{(d)}_2 \in \mathbb{R}^{K^{(d)}_1 \times K^{(d)}_2},
\bs \psi^{(d)}_3 \in \mathbb{R}^{K^{(d)}_2 \times K}$ are parameter matrices.
The neural network size is specified by $p_d, K_1^{(d)},K_2^{(d)}, K \in \mathbb{N}$ $(d=1,2,\ldots,D)$.
Each element of $\bs \sigma(\x)$ is user-specified activation function $\sigma(\cdot)$. Although we basically consider a multi-layer perceptron (MLP) as $f^{(d)}_{\bs \psi}$, it can be replaced with any deterministic NN with input and output layers, such as recurrent NN and Deep NN~(DNN).

%\begin{remark}
%\normalfont
%\label{rem:linear_model}
NN model reduces to linear model
\begin{align}
f^{(d)}_{\bs \psi}(\x):=\bs \psi^{(d)\top}\x
\label{eq:linear_transformation}
\end{align}
 $(d=1,2,\ldots,D)$
by applying $\sigma(x)=x$, where $\bs \psi^{(d)} \in \mathbb{R}^{p_d \times K}$ is  parameter matrix.
%\end{remark}

%\textbf{Linear transformation} is defined as
%\begin{align}
%f^{(d)}_{\bs \psi}(\x):=\bs \psi^{(d)\top}\x,
%\label{eq:linear_transformation}
%\end{align}
%where $\x \in \mathbb{R}^{p_d}$ is a data vector, $\bs \psi^{(d)} \in \mathbb{R}^{p_d \times K}$ is a parameter matrix to be estimated $(d=1,2,\ldots,D)$, and $K \leq \min_d p_d$ is user-specified. 
%Eq.~(\ref{eq:linear_transformation}) can be extended to affine transformation $\tilde{\bs \psi}^{(d)\top}\tilde{\x}=\bs \psi^{(d)\top}\x+\bs \phi^{(d)}$ by replacing $\x,\bs \psi^{(d)}$ with 
%$\tilde{\x}:=(\x^{\top},1)^{\top},\tilde{\bs \psi}^{(d)}:=(\bs \psi^{(d)\top},\bs \phi^{(d)})^{\top}$, where $\bs \phi^{(d)} \in \mathbb{R}^{K}$ is a parameter vector for the translation. 
%Note that NN reduces to linear transformation if the identity functions $\sigma(x)=x$ are used as activation functions.

\subsection{Link weights across some view pairs may be missing}
\label{subsec:view_missing}
Link weights across all the view pairs may not be available in practice.
%for all pairs of views to be associated, 
So we consider the set of unordered view pairs
%of view pairs whose link weights are observed as 
%特定のview間にリンクが張られている状況を考える．本論文では特に，
\begin{align*}
\mathcal{D}&:=\{(d,e) \text{ : Link weights are} 
\\
&\hspace{7em}\text{observed between views $d,e$}\},
\end{align*}
and we formally set $w_{ij}=0$ for the missing $(d_i,d_j) \notin \mathcal{D}$.
For example, $\mathcal{D}=\{(1,2)\}$ for $D=2$ indicates that link weights across view-1 and view-2 are observed while link weights within view-1 or view-2 are missing. 
We should notice the distinction between setting $w_{ij}=0$ with missing and observing $w_{ij}=0$ without missing, because these two cases give different likelihood functions.
%Since $\mathcal{D}$ is fixed, in the following, we do not write $\mathcal{D}$ explicitly unless necessary. 

\subsection{Maximum Likelihood Estimator}
\label{subsec:mle}

Since PMvGE is specified by (\ref{eq:w_model}) and (\ref{eq:mu_ij}),  the log-likelihood function is given by
\begin{align}
{\ell}_n(\bs \alpha,&\bs \psi) \nonumber\\
:=& 
\sum_{(i,j) \in \mathcal{I}_n} 
[
	w_{ij}\log \mu_{ij}(\bs \alpha,\bs \psi)-\mu_{ij}(\bs \alpha,\bs \psi)
],
\label{eq:log_likelihood}
\end{align}
whose sum is over the set of index pairs $\mathcal{I}_n:=\{(i,j) \mid 1 \leq i < j \leq n,(d_i,d_j) \in \mathcal{D}\}$. 
The Maximum Likelihood Estimator~(MLE) of the parameter $(\bs \alpha,\bs \psi)$ is defined as the maximizer of (\ref{eq:log_likelihood}).
%The probability of $\{w_{ii}\}_{i=1}^{n}$ is not considered since the links $\{(v_i, v_i)\}$ are not observed in practice. 
We estimate $(\bs \alpha,\bs \psi)$ by maximizing ${\ell}_n(\bs \alpha,\bs \psi)$ with constraints $\bs \alpha=\bs \alpha^{\top}$ and $\alpha^{(d,e)}=0$ for $(d,e) \notin \mathcal{D}$.

\subsection{PMvGE approximately generalizes various methods for multi-view learning}
\label{subsec:relation}

SBM (\ref{eq:sbm_mu_ij}) is 1-view PMvGE with 1-hot cluster membership vector $\x_i \in \{0,1\}^C$.
Consider the linear model (\ref{eq:linear_transformation}) with 
$\bs \psi^{(1)\top} = (\bs \psi^1,\ldots, \bs \psi^C) \in \mathbb{R}^{K \times C}$, where $\bs \psi^c \in \mathbb{R}^K$ is a feature vector for class $c=1,\ldots,C$.
Then $\mu_{ij} = \alpha^{(1,1)} \exp( \langle \bs \psi^{c_i}, \bs \psi^{c_j} \rangle)$ 
 will specify $\beta^{(c_i,c_j)}$ for sufficiently large $K$, and
 $\{\bs \psi^{c_i}\}$ represent low-dimensional structure of SBM for smaller $K$.
SBM is also interpreted as PMvGE with $D=C$ views by letting $d_i = c_i$ and $f^{(d_i)}_{\bs \psi}(\x)\equiv 0$. Then $\mu_{ij}=\alpha^{(c_i,c_j)}$ is equivalent to SBM.

More generally, CDMCA~\citep{shimodaira2016cross} is approximated by $D$-view PMvGE with the linear transformations (\ref{eq:linear_transformation}).
The first half of the objective function~(\ref{eq:log_likelihood}) becomes
\begin{align}
\frac{1}{2}
\sum_{i=1}^{n}\sum_{j=1}^{n}w_{ij}
\langle \bs \psi^{(d_i)\top}\x_i,\bs \psi^{(d_j)\top}\x_j \rangle
\label{eq:cdmca_objective}
\end{align}
by specifying $\alpha^{(d,e)}\equiv1$.
CDMCA computes the transformation matrices $\{\bs \psi^{(d)}\}$ by maximizing this objective function under a quadratic constraint such as 
\begin{align}
\sum_{i=1}^{n}\sum_{j=1}^{n}w_{ij}\bs \psi^{(d_i)\top}\x_i \x_i^{\top} \bs\psi^{(d_i)}=\bs I.
\label{eq:cdmca_constraint}
\end{align}
The above observation intuitively explains why PMvGE approximates CDMCA; this is more formally discussed in 
Supplement~\ref{sec:optimization_linear_transformation}.
%Section~\ref{subsec:optimization_linear_transformation}.
The quadratic constraint is required for preventing the maximizer of (\ref{eq:cdmca_objective}) from being diverged in CDMCA. However, our likelihood approach does not require it, because the last half of~(\ref{eq:log_likelihood}) serves as a regularization term.

\subsection{PMvGE represents neural network classifiers of multiple classes}
\label{subsec:pmvge_classifiers}
For illustrating the generality of PMvGE, here we consider the multi-class classification problem. 
We show that PMvGE includes Feed-Forward Neural Network~(FFNN) classifier as a special case.

Let $(\bs x_i, c_i) \in \mathbb{R}^p \times  \{1,\ldots,C\}$, $i=1,\ldots, n$ be the training data for the classification problem of $C$ classes.
FFNN classifier with softmax function~\citep{prml} is defined as
\begin{align*}
h_j(\bs x_i)
:=
\frac{\exp(f_{\bs \psi,j}(\bs x_i))}{\sum_{j=1}^{C}\exp(f_{\bs \psi,j}(\bs x_i))},\quad
j=1,\ldots, C,
\end{align*}
where $f_{\bs \psi}(\bs x):=(f_{\bs \psi,1}(\bs x),\ldots,f_{\bs \psi,C}(\bs x)) \in \mathbb{R}^{C}_{\geq 0}$ is a multi-valued neural network, and $\bs x_i$ is classified into the class $\displaystyle \argmax_{j \in \{1,2,\ldots,C\}} h_j(\bs x_i)$. This classifier is equivalent to 
\begin{align}
\argmax_{j \in \{1,2,\ldots,C\}}
\exp\left(
	\langle f_{\bs \psi}(\bs x_i),\bs e_j \rangle
\right),
\label{eq:classifier}
\end{align}
where $\bs e_j \in \{0,1\}^C$ is the 1-hot vector with 1 at $j$-th entry and 0 otherwise. 

The classifier~(\ref{eq:classifier}) can be interpreted as PMvGE with $D=2,\mathcal{D}=\{(1,2)\}$ as follows.
For view-1, $f^{(1)}_{\bs \psi}(\bs x):=f_{\bs \psi}(\bs x)$ and inputs are
$\x_1,\ldots,\x_n$.
For view-2, $f^{(2)}_{\bs \psi}(\bs x'):=\bs x'$ and inputs are
$\bs e_1,\ldots, \bs e_C$.
We set $w_{ij}=1$ between $\x_i$ and $\bs e_{c_i}$, and $w_{ij}=0$ otherwise.

%%(ICMLで新規に挿入したsection．修正必要)
%Feedforward neural network classifier of $C$ classes can be interpreted as a special case of PMvGE by taking the 1-hot vector of $C$ classes as the input $\bs x_j$ of the second view, because $\bs a_j^{\top} f_{\bs \psi}(\x_i) =\langle f_{\bs \psi}(\bs x_i), \bs A^{\top} \bs x_j \rangle$, where $f_{\bs \psi}(\bs x_i)$ is the feature vector of neural network and $\bs A=(\bs a_1,...,\bs a_C)^{\top}$ is $C \times K$ weight matrix of the final layer (slight modification of the regularization term is necessary for the equivalence). This also shows a generality of PMvGE.

\section{Optimization}
\label{sec:optimization}
In this section, we present an efficient way of optimizing the parameters for maximizing 
the objective function ${\ell}_n(\bs \alpha,\bs \psi)$ defined in eq.~(\ref{eq:log_likelihood}).
We alternatively optimize the two parameters $\bs \alpha$ and $\bs \psi$.
Efficient update of $\bs \psi$ with minibatch SGD is considered in Section~\ref{subsec:scalable_learning_general}, and 
update of $\bs \alpha$ by solving an estimating equation is considered in Section~\ref{subsec:update_alpha}.
We iterate these two steps for maximizing ${\ell}_n(\bs \alpha,\bs \psi)$.

%and we summarize these in Section~\ref{subsec:scalable_learning_general}. 
%Besides, especially for Linear transformation, another algorithm based on Local Quadratic Approximation~(LQA) can be derived. Each step of the algorithm corresponds to CDMCA. 
%This result is shown in Section~\ref{subsec:optimization_linear_transformation}. 

%
\subsection{Update of $\bs \psi$}
\label{subsec:scalable_learning_general}
%\textbf{Scalable learning for general setting:}
We update $\bs \psi$ using the gradient of ${\ell}_n(\bar{\bs \alpha},\bs \psi)$ by fixing current parameter value $\bar{\bs \alpha}$.
Since $\W=(w_{ij})$ may be sparse in practice, the computational cost of minibatch SGD~\citep{Goodfellow-et-al-2016-Book} for the first half of ${\ell}_n(\bar{\bs \alpha},\bs \psi)$ is expected to be reduced by considering the sum over the set $\mathcal{W}_n:=\{(i,j) \in \mathcal{I}_n \mid w_{ij}>0\}$. 
On the other hand, there should be $|\mathcal{I}_n| =O(n^2)$ positive terms in the last half of ${\ell}_n(\bar{\bs \alpha},\bs \psi)$, so we consider the sum over node pairs uniformly-resampled from $\mathcal{I}_n$.

We make two sets $\mathcal{I}'_n,\mathcal{W}'_n$ by picking $(i,j)$ from $\mathcal{I}_n$ and $\mathcal{W}_n$, respectively, so that 
$$
|\mathcal{I}_n'|+|\mathcal{W}'_n|=m, \:
|\mathcal{I}_n'|/|\mathcal{W}'_n|=r.
$$ 
User-specified constants $m \in \mathbb{N}$ and $r>0$ are usually called as ``minibatch size" and ``negative sampling rate". 
We sequentially update minibatch $\mathcal{I}'_n,\mathcal{W}'_n$
for computing the gradient of 
%By alternating the resampling of $(i,j)$ from $\mathcal{I}_n(\mathcal{D}),\mathcal{W}_n(\mathcal{D})$ and computing the gradient of 
\begin{align}
\scalebox{0.9}{$\displaystyle
	\sum_{(i,j) \in \mathcal{W}_n'} w_{ij} \log \mu_{ij}(\bar{\bs \alpha},\bs \psi)
-
\tau
\sum_{(i,j) \in \mathcal{I}_{n}'} \mu_{ij} (\bar{\bs \alpha},\bs \psi)
$}
\label{eq:minibatch_gradient}
\end{align}
with respect to $\bs \psi$. By utilizing the gradient, the parameter $\bs \psi$ can be sequentially updated by SGD, where $\tau>0$ is a tuning parameter.
Eq.~(\ref{eq:minibatch_gradient}) approximates ${\ell}_n(\bar{\bs \alpha},\bs \psi)$ if $(i,j)$ are uniformly-resampled and $\tau=|\mathcal{I}_n|/(r \, |\mathcal{W}_n|)$, however, smaller $\tau$ such as $\tau=1$ may make this algorithm stable in some cases.

\subsection{Update of $\bs \alpha$} 
\label{subsec:update_alpha}
Let $\bar{\bs \psi}$ represent current parameter value of $\bs \psi$.
By solving the estimating equation $\frac{\partial {\ell}_n(\bs \alpha,\bar{\bs \psi})}{\partial \alpha^{(d,e)}}=0$ with respect to $\alpha^{(d,e)}$ under constraints $\bs \alpha=\bs \alpha^{\top}$ and $\alpha^{(d,e)}=0$, $(d,e) \notin \mathcal{D}$, we explicitly obtain a local maximizer of ${\ell}_n(\bs \alpha,\bar{\bs \psi})$. 
However, the local maximizer requires roughly $O(n^2)$ operations for computation. 
To reduce the high computational cost, we efficiently update $\bs \alpha$ by (\ref{eq:alpha_update}), which is a minibatch-based approximation of the local maximizer. 
%$\bar{\bs \psi}$を定数として$\bs \alpha$に関する推定方程式を解けば，
\begin{align}
\scalebox{0.9}{$\displaystyle
	\hat{\alpha}^{(d,e)}
	:=
	\begin{cases}
\dfrac{		\sum_{(i,j) \in \mathcal{I}_n^{'(d,e)}}w_{ij}}{
		\sum_{(i,j) \in \mathcal{I}_n^{'(d,e)}}\exp(\bar{g}_{ij})
	}
	&\text{for } (d,e) \in \mathcal{D} \\
	0 & \text{otherwise}, \\
	\end{cases}
$}
\label{eq:alpha_update}
\end{align}
where $\bar{g}_{ij}:=\langle f^{(d)}_{\bar{\bs \psi}}(\x_i),f^{(e)}_{\bar{\bs\psi}}(\x_j) \rangle$, $\mathcal{I}_n^{'(d,e)}:=\{(i,j) \in \mathcal{I}'_n \mid (d_i,d_j)=(d,e)\}$, and $\mathcal{I}'_n$ is the minibatch defined in Section~\ref{subsec:scalable_learning_general}. 
Note that $(d,e)$ is unordered view pair.
%Thus the minibatch size regulates the tradeoff between computational complexity and the accuracy. 

\subsection{Computational cost}
PMvGE requires $O(m)$ operations for each minibatch iteration.
It is efficiently computed even if the number of data vectors $n$ is very large.

%
%\begin{algorithm}[htbp]                      
%\caption{Probabilistic Multi-view Graph Embedding~(PMvGE)}         
%\label{alg:minibatch SGD}                          
%\begin{algorithmic}                  
%\REQUIRE $d_i$: view of the node $v_i$, $\bs x_i \in \mathbb{R}^{p_{d_i}}$: data vector of node $v_i$, $\eta_t>0$: learning rate at step $t$, $m \in \mathbb{N}$: minibatch size, $r>0$: negative sampling rate, $T \in \mathbb{N}$: number of iteration, $\mathcal{D}$: pairs of views to be considered, 
%\ENSURE $w_{ij} = 0$ if $(d_i,d_j) \notin \mathcal{D}$
%\STATE $\bs \alpha \Leftarrow \bs O$
%\STATE $\bs \psi$: randomly initialized
%\FOR{$t=1$ to $T$}
%\STATE $\bar{\bs \psi} \Leftarrow \bs \psi-\eta_t \cdot \frac{\partial \text{eq.}(\ref{eq:minibatch_gradient})}{\partial \bs \psi}$
%\STATE $\bs \alpha \Leftarrow$ Eq.~(\ref{eq:alpha_update})
%\STATE $\bar{\bs \alpha} \Leftarrow \bs \alpha$
%\ENDFOR
%\RETURN $(\bs \alpha,\bs \psi)$
%\end{algorithmic}
%\end{algorithm}
%

\section{PMvGE learns arbitrary similarity measure}
\label{sec:theoretical_properties}
Two theoretical results are shown here for indicating that PMvGE with sufficiently large neural networks
 learns arbitrary similarity measure using sufficiently many data vectors.
In Section~\ref{subsec:case_neural_network}, 
we prove that arbitrary similarity measure can be approximated by the inner product in the shared space
with sufficiently large neural networks.
%This result suggests a high expressive power of PMvGE. 
% if the number of neurons in neural networks is large. 
In Section~\ref{subsec:consistency_pmvge}, we prove that
MLE of PMvGE converges to the true parameter value, i.e., the consistency of MLE, in some sense as the number of data vectors increases.
%MLE in PMvGE approaches to the true parameter in some sense, as the number of nodes $n$ goes to infinity. 

\subsection{Inner product of NNs approximates a wide class of similarity measures across views}
\label{subsec:case_neural_network}
%本節では，簡単のために$\mathcal{D}=\{(1,2),(2,1)\}$, meaning that we only consider inter-relationships across $D=2$ views. 

Feedforward neural networks with ReLU or sigmoid function are proved to be able to approximate arbitrary continuous functions under some assumptions~\citep{cybenko1989approximation,funahashi1989approximate,yarotsky2016error,pmlr-v70-telgarsky17a}.  
%シグモイド関数やReLUユニットを活性化に用いたfeed forward neural networksはある閉集合上で任意の連続関数を近似できることが知られている~\citep{cybenko1989approximation,yarotsky2016error,pmlr-v70-telgarsky17a}．
However, these results cannot be directly applied to PMvGE, because our model is based on the inner product of two neural networks $\langle f_{\bs \psi}^{(d_i)}(\bs x_i),f_{\bs \psi}^{(d_j)}(\bs x_j)\rangle$. 
%しかしながら，我々はNNの出力$f^{(d)}_{\bs \psi}(\bs x)$をそのまま利用するのではなく，NNの内積を用いた類似度$\langle f_{\bs \psi}^{(1)}(\bs x_i),f_{\bs \psi}^{(2)}(\bs x_j)\rangle$を用いて$w_{ij}$の確率モデルを定義するため，既存の定理を直接適用することができない．
In Theorem~\ref{theo:universal_approximate}, we show that the inner product can approximate 
$
g_*\left(
    f_*^{(d_i)}(\bs x_i),f_*^{(d_j)}(\bs x_j)
\right)
$, that is,
arbitrary similarity measure $g_* (\cdot, \cdot)$ in $K_*$ dimensional shared space, where
$f_*^{(d_i)}$, $f_*^{(d_j)}$ are arbitrary two continuous functions.
For showing 
$g_*( f_*^{(d)}(\bs x), f_*^{(e)}(\bs x') ) \approx \langle f_{\bs \psi}^{(d)}(\bs x), f_{\bs \psi}^{(e)}(\bs x')  \rangle$,
the idea is first consider a feature map $\bs \Phi_K : \mathbb{R}^{K^*} \to \mathbb{R}^K$
for approximating
$g_*(\bs y_*,\bs y'_*)  \approx \langle \bs \Phi_K(\bs y_*),\bs \Phi_K(\bs y_*')\rangle$ with sufficiently large $K$,
and then consider neural networks $f_{\bs \psi}^{(d)}: \mathbb{R}^{p_d} \to \mathbb{R}^K$ with sufficiently many hidden units
for approximating $ \bs \Phi_K(f_*^{(d)}(\bs x)) \approx f_{\bs \psi}^{(d)}(\bs x)$.

%そこで我々は，2つのneural networkの内積がより広いクラスの類似度

% where $g_*(\cdot,\cdot)$ is an arbitrary $C^1$-function and $f^{(1)}_*,f^{(2)}_*$ are arbitrary continuous functions
%を近似できることを定理~\ref{theo:universal_approximate}として示す．

%Consistency in general settings is derived in Section~\ref{subsec:consistency_general}. 
%NN satisfies the regularity condition [C-1] as long as well-known activation functions such as ReLU $\sigma(x)=\max\{0,x\}$, Linear $\sigma(x)=x$, Logistic $\sigma(x)=1/(1+\exp(-\alpha x))$ or Hypobolic tangent $\sigma(x)=\tanh(x)$, are used.  
%%活性化関数のうち，ReLU $\sigma(x)=x_+$, Linear $\sigma(x)=x$, Logistic $\sigma(x)=1/(1+\exp(-\alpha x))$, Hyp.~tan. $\sigma(x)=\tanh(x)$などはコンパクト集合上で1-Lipschitzであるから，これらを利用したとき$\{f^{(d)}_{\bs \psi}\}$は[C-1]を満たす．
%Therefore, the consistency of NN can be obtained by Theorem~\ref{theo:consistency_general}, if the condition [C-2] holds. 
%%[C-3]:真の解が解空間に含まれる，を加えると，Theorem~\ref{theo:consistency_general}より, 
%%NNを用いた類似度のconsistencyが得られる．
%%したがって，データ数が増加するに従って平均パラメータが正確に推定できることが分かる．
%However, 
%It is not obvious if the function $\mu_{ij}(\bs \alpha,\bs \psi)$, which is composed of the inner product similarity, can explain the real-world data structure. 
%一方で，類似度関数$g$の形は実際には分からず，misspecifyしていると考えるべきである． 

\begin{theo}
\normalfont
\label{theo:universal_approximate}
Let $f^{(d)}_*:[-M,M]^{p_d}\to[-M',M']^{K_*}$, $d=1,2,\ldots,D$,
be continuous functions and $g_*:[-M',M']^{K_*} \times [-M',M']^{K_*} \to \mathbb{R}$ be a symmetric, continuous, and positive-definite kernel function for some $K^*, M,M'>0$. 
$\sigma(\cdot)$ is ReLU or activation function which is non-constant, continuous, bounded, and monotonically-increasing.
Then, for arbitrary $\varepsilon>0$, 
by specifying sufficiently large $K \in \mathbb{N},T=T(K) \in \mathbb{N}$, 
there exist $\bs A^d \in \mathbb{R}^{K \times T},\bs B^d \in \mathbb{R}^{T \times p_d},\bs c \in \mathbb{R}^{T}$, $d \in \{1,2,\ldots,D\}$, such that
\begin{equation}
\scalebox{0.9}{$
\displaystyle 
\bigg|
g_*\left(f^{(d)}_*(\x),f^{(e)}_*(\x')\right)
-
\big\langle f_{\bs \psi}^{(d)}(\x), f_{\bs \psi}^{(e)}(\x') \big\rangle
\bigg|
<\varepsilon
$}
\end{equation}
for all $(\bs x,\bs x') \in [-M,M]^{p_d + p_e}$,
$d,e \in \{1,2,\ldots,D\}$,
 where $f_{\bs \psi}^{(d)}(\x^d)
=
\bs A^d \bs \sigma(\bs B^d\x^d + \bs c^d)$, $d=1,2,\ldots,D$, are two-layer neural networks with $T$ hidden units
and  $\bs\sigma(\x)$ is element-wise $\sigma(\cdot)$ function.
\end{theo}

Proof of the theorem is given in Supplement~\ref{sec:univ_supp}.

%\begin{theo}
%\normalfont
%\label{theo:universal_approximate_kernel}
%We assume the same conditions as in Theorem~\ref{theo:universal_approximate} for $d=1$.
%We also assume that $g_*(\cdot,\cdot)$ is positive definite kernel.
%Then, for arbitrary $\varepsilon>0$, 
%by specifying sufficiently large $K \in \mathbb{N},N_1 \in \mathbb{N}$, 
%there exist $\bs A^1 \in \mathbb{R}^{K \times N_1},\bs B^1 \in \mathbb{R}^{N_1 \times p_1},\bs c \in \mathbb{R}^{p_1}$ such that
%$$
%\scalebox{0.9}{$
%\displaystyle 
%\bigg|
%g_*\left(f^{(1)}_*(\x),f^{(1)}_*(\x')\right)
%-
%\big\langle \hat{f}^{(1)}(\x), \hat{f}^{(1)}(\x') \big\rangle
%\bigg|
%<\varepsilon
%$}
%$$
%for all $(\bs x,\bs x') \in [-M,M]^{2p_1}$.
%\end{theo}

%Theorem~\ref{theo:universal_approximate}は$D=1$のとき
If $D=1$, Theorem~\ref{theo:universal_approximate} corresponds to Mercer's theorem~\citep{mercer1909functions,courant89} of Kernel methods, which states that arbitrary positive definite kernel $g_*(\cdot,\cdot)$ can be expressed as the inner product of high-dimensional feature maps. 
While Mercer's theorem indicates only the existence of such feature maps, Theorem~\ref{theo:universal_approximate} also states that the feature maps can be approximated by neural networks.

%Interestingly, the assumption of positive-definiteness of $g_*$ in the theorem can be dropped to satisfy (\ref{eq:universal_approximate}) for each cross-view case $d \neq e$ rather than all the view pairs simultaneously.
%This means that $2$-view PMvGE with bipartite graph approximates more general class of similarity measures.

%任意の$\bs x_1,\bs x_2 \in \mathcal{X}$について正定値カーネル$g_*(\bs x_1,\bs x_2)$が特徴量$\bs x_1,\bs x_2$のある非線形変換の内積で表示できることを示したMercerの定理~\citep{mercer1909functions,courant89} of Kernel methodsに対応する．
%Unlike the mercer's theorem which describes only the existence of feature maps, our model automatically approximates the feature maps by neural networks.
%Mercerの定理ではカーネルを内積で表現する非線形変換の存在を示しているだけで具体的な構成法に触れていないが，我々の方法では非線形変換そのものをニューラルネットワークで近似できる．

%$D=1$の場合のMercerの定理と比べて，$D=2$の場合を扱う我々の定理~\ref{theo:universal_approximate}には2つの違いがある:
%\begin{enumerate}[{(i)}]
%\item Mercerの定理ではカーネルを内積で表現する非線形変換の存在を示しているだけで具体的な構成法に触れていないが，我々の方法では非線形変換そのものをニューラルネットワークで近似できる．
%\item $D=1$のとき$g_*$には正定値性などの条件が必要であったが，$D=2$のときは$C^1$級という性質のみしか仮定しない．これは例えば，実数$x,y \in \mathbb{R}$について内積$x^2$がpositiveな値しか取れないことに比べて，$xy$は実数域全体をとれることとのanalogyになっている．
%\end{enumerate}

\textbf{Illustrative example} 
As an example of positive definite similarity, we consider the cosine similarity 
\[
g_*(f_*(\bs x),f_*(\bs x'))
:=
\frac{
\langle f_*(\bs x),f_*(\bs x') \rangle
}{\|f_*(\bs x)\|_2 \|f_*(\bs x')\|_2}	
\]
with 
$f_*(\bs x)=(x_1,\cos x_2,\exp(-x_3),\sin(x_4-x_5)),p=5$. 
%%Here, let us consider the situation that our model~(\ref{eq:mu_ij}) is misspecified. 
%%(まず例から始める，という一文を入れる)
%For the simplicity, suppose that $D=2,p_1=3,p_2=4$, and data vectors $\x=(x_{1},x_{2},x_{3}) \in \mathbb{R}^{3}$ in view-1 and $\x'=(x'_{1},x'_{2},x'_{3},x'_{4})$ in view-2 are transformed into $2$-dimensional common space by 
%$f^{(1)}_*(\x):=(x_1,\sin x_2-x_3), 
%f^{(2)}_*(\x'):=(\sin x'_1x'_2,\cos (x'_3-x'_4))$. 
%We assume that the similarity of two different types of vectors $\x$ in view-1 and $\x'$ in view-2 is measured by the square of the euclidean distance $\|f^{(1)}_*(\x)-f^{(2)}_*(\x')\|_2^2$. 
%If the similarity is used in the model~(\ref{eq:mu_ij}) instead of the inner-product, our model with the inner-product should be misspecified. 
%However, we prove in Theorem~\ref{theo:universal_approximate} that the inner-product of feature vectors computed by neural networks $\hat{f}^{(1)}_K(\x),\hat{f}^{(2)}_K(\x')$ whose outputs are $K$-dimensional
%%\begin{align}
%$
%\langle \hat{f}^{(1)}_K(\x),\hat{f}^{(1)}_K(\x') \rangle
%$
%%\label{eq:inner_product_similarity}
%%\end{align}
%approximates a wide range of similarities such as eq.~(\ref{eq:l2_similarity}) within a given accuracy. 
%In Figure~\ref{fig:exp_uat}, heat maps of the square of the euclidean distance 
For 2-dim visualization in Fig.~\ref{fig:exp_uat} with $(s,t)\in\mathbb{R}^2$, let us define
$
G_*(s,t):=
g_*(f_*(s \bs e_1),f_*(t \bs e_2)),
\bs e_1:=(1,1,1,0,0),
\bs e_2:=(0,0,1,1,1)$ and its approximation by the inner product of neural networks $
\hat{G}_K(s,t):=\langle f_{\bs \psi}(s\bs e_1), f_{\bs \psi}(t\bs e_2) \rangle
$ with $T = 10^3$ hidden units.
If $K$ and $T$ are sufficiently large, $\hat{G}_K(s,t)$ 
approximates $G_*(s,t)$ very well 
as suggested by Theorem~\ref{theo:universal_approximate}.

\begin{figure}[htbp]
    \begin{center}
        \subfigure[$\hat{G}_2(s,t)$]{
        \label{fig:h2}
            \includegraphics[width=2.2cm]{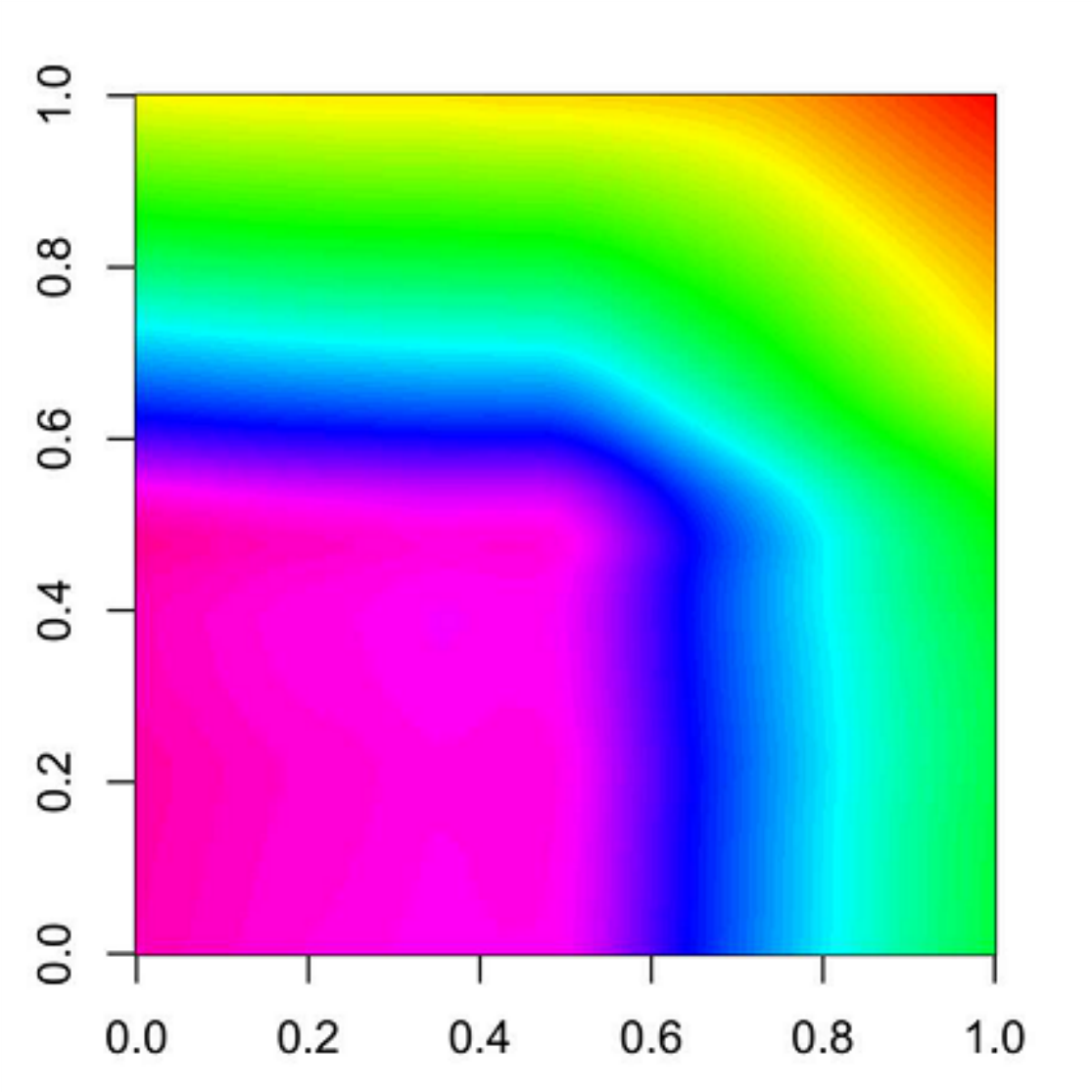}}
        \subfigure[$\hat{G}_{50}(s,t)$]{
        \label{fig:h100}
            \includegraphics[width=2.2cm]{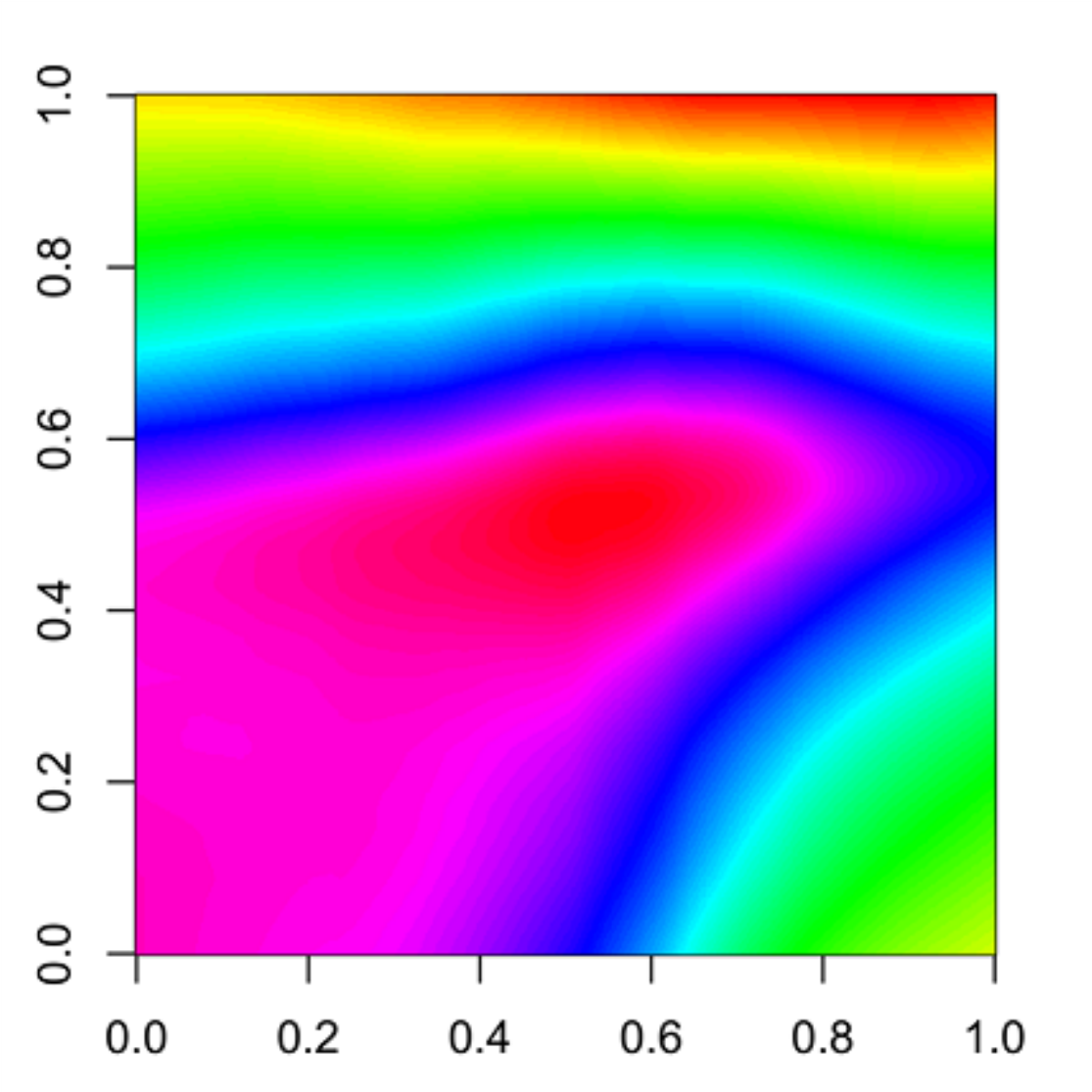}}
        \subfigure[$G_*(s,t)$]{
        \label{fig:htrue}
           \includegraphics[width=2.2cm]{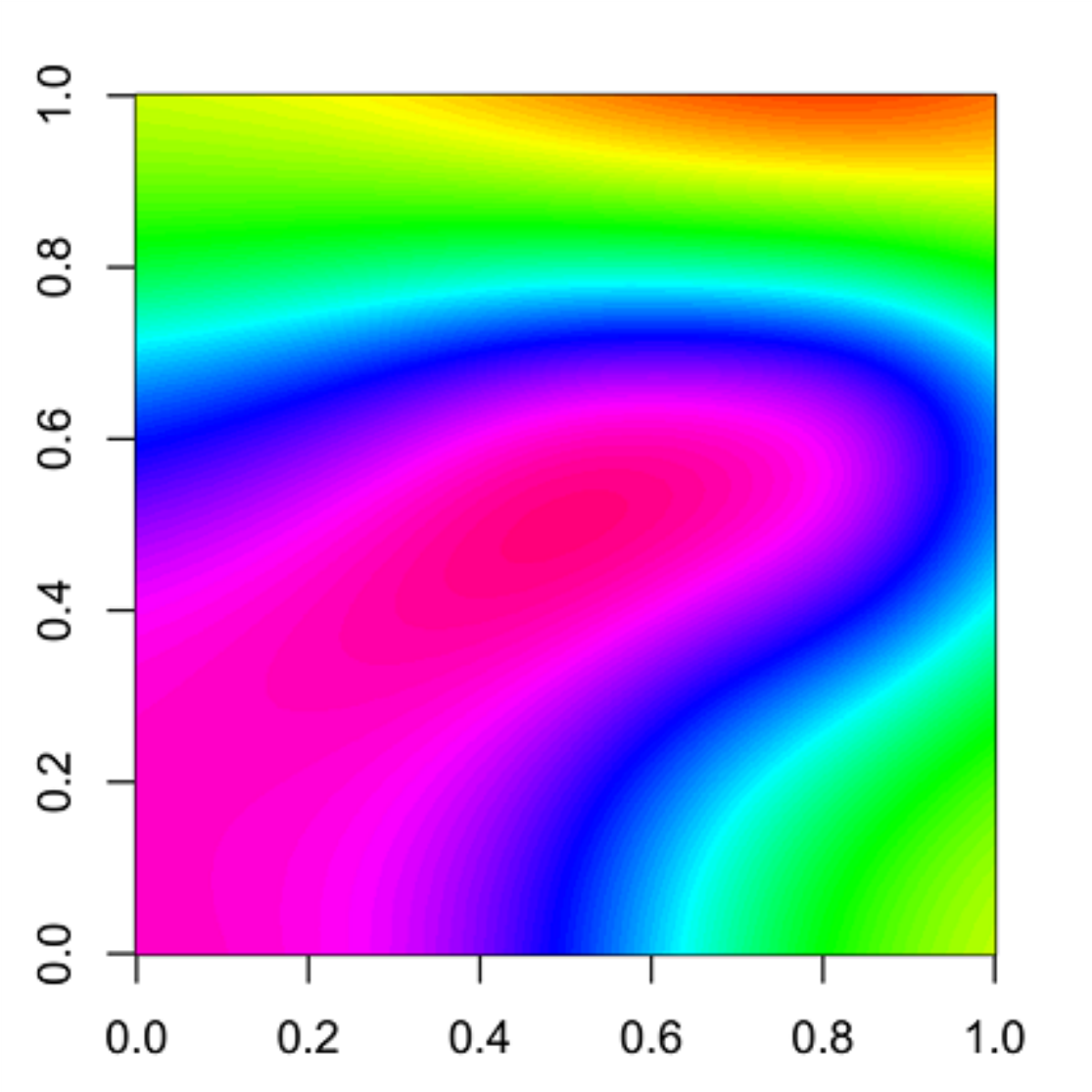}}
        \includegraphics[width=0.8cm]{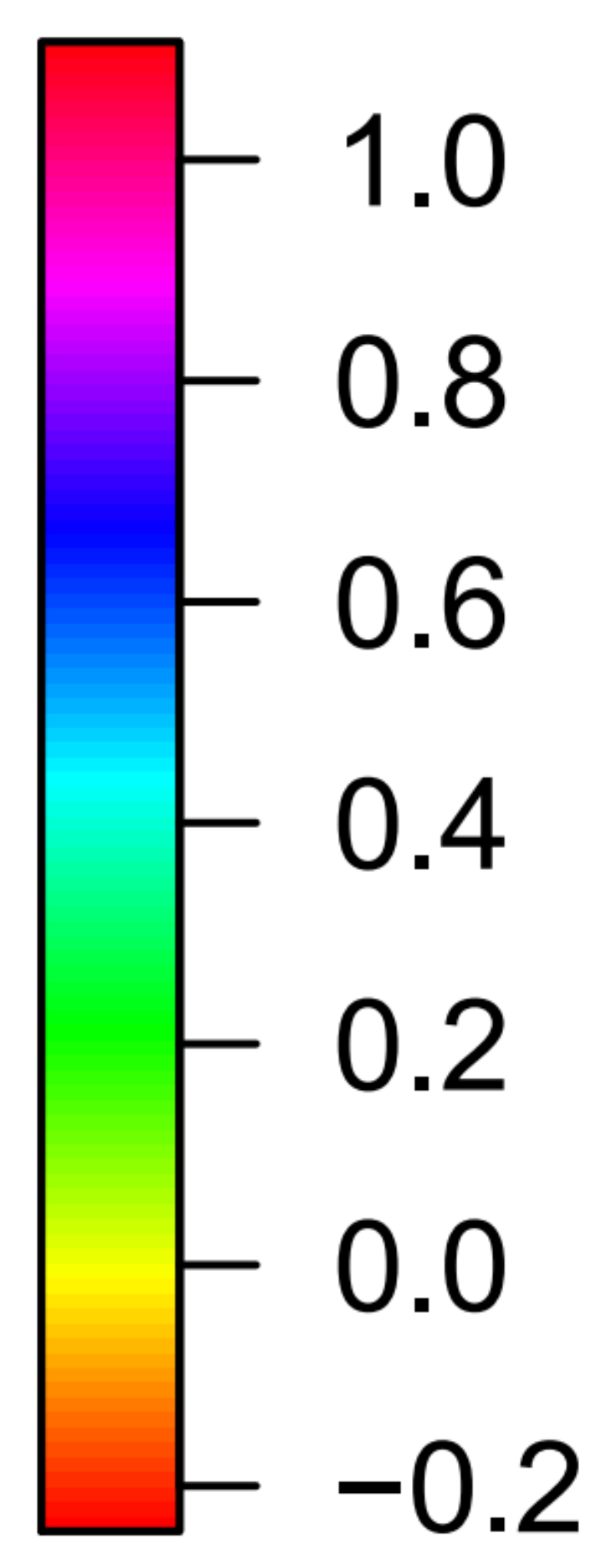}
    \end{center}   
    \vspace{-1em}
\caption{A two-dim visualization $G_*(s,t)$ of similarity measure $g_*(\bs y_*, \bs y_*') =\frac{ \langle \bs y,\bs y' \rangle }{\|\bs y\|_2 \|\bs y'\|_2}$ with $K_*=2$ is well approximated by the visualization $\hat G_K(s,t)$ of the inner product
$\langle \bs y, \bs y'\rangle$ for $K=100$, while the approximation is poor for $K=2$.
A neural network $f_{\bs \psi}$ with $T=10^3$ ReLU units and $K=2$ or $K=50$ linear output units
are trained with $n=10^3$ data vectors and link weights across views.
}
\label{fig:exp_uat} 
\end{figure}

\subsection{MLE converges to the true parameter value} 
\label{subsec:consistency_pmvge}

We have shown the universal approximation theorem of similarity measure in Theorem~\ref{theo:universal_approximate}.
However, it only states that the good approximation is achieved if we properly tune the parameters of neural networks.
Here we argue that MLE of Section~\ref{subsec:mle} will actually achieve the good approximation if we have sufficiently many data vectors.
%In this section, we describe the consistency of MLE in PMvGE, by considering the generative model of $(\bs x_i,d_i)$. 
%With fixed $p_d$ $(d=1,2,\ldots,D)$, we prove that an MLE $\hat{\bs \theta}_n:=(\hat{\bs \alpha}_n,\hat{\bs \psi}_n)$ converges in some sense to the true parameter $\bs \theta_*:=(\bs \alpha_*,\bs \psi_*)$ as $n \to \infty$. 
The technical details of the argument are given in Supplement~\ref{sec:consistency_supp}.

%We suppose following assumptions: 
%for all $i=1,2,\ldots,n$ and $d=1,2,\ldots,D$, 
%(1) $\bs x_i$ independently follows a distribution $Q^{(d_i)}$ whose support $\mathcal{X}$ is compact, 
%(2) view indicator $d_i$ independently follows a probability distribution $\mathbb{P}(d_i=d)=\eta^{(d)} \in (0,1)$, 
%(3) transformation $f^{(d)}_{\bs \psi}(\bs x)$ is Lipschitz continuous with respect to $(\bs \psi,\bs x)$. 
%(4) with a small $\delta>0$, parameter set $\bs \Theta$ is a compact set $[\delta,1/\delta]^{D'} \times \bs \Psi$ which includes a true parameter $\bs \theta_*:=(\bs \alpha_*,\bs \psi)_*$.

Let $\bs \theta \in \bs \Theta$ denote the vector of free parameters in $\bs \alpha, \bs \psi$,
and $\ell_n(\bs \theta)$ be the log-likelihood function (\ref{eq:log_likelihood}).
We assume that the optimization algorithm in Section~\ref{sec:optimization} successfully
computes MLE $\hat {\bs \theta}_n$ that maximizes $\ell_n(\bs \theta)$.
Here we ignore the difficulty of global optimization, while we may only get a local maximum in practice.
We also assume that PMvGE is correctly specified; there exists a parameter value $\bs \theta_*$ so that the parametric model represents the true probability distribution.

Then, we would like to claim that $\hat {\bs \theta}_n$ converges to the true parameter value $\bs \theta_*$ in the limit of $n\to\infty$, the property called the consistency of MLE.
However, we have to pay careful attention to the fact that PMvGE is not a standard setting
in the sense that (i)~there are correlated $O(n^2)$ samples instead of $n$ i.i.d.~observations, and (ii)~the model is not identifiable with infinitely many equivalent parameter values; for example there are rotational degrees of freedom in the shared space so that $\langle \bs y, \bs y'\rangle = \langle \bs O \bs y, \bs O \bs y'\rangle$ with any orthogonal matrix $\bs O$ in $\mathbb{R}^K$.
We then consider the set of equivalent parameters
$\hat{\bs \Theta}_n:=\{\bs \theta \in \bs \Theta \mid \ell_n(\bs \theta)=\ell_n(\hat{\bs \theta}_n)\}$.
Theorem~\ref{theo:consistency_general} states that, as $n\to\infty$,
$\hat{\bs \Theta}_n$ converges to $\bs \Theta_*$, the set of $\bs \theta$ values equivalent to $\bs \theta_*$.

%Then, it can be proved that $\frac{1}{|\mathcal{I}_n(\mathcal{D})|}{\ell}_n(\bs \theta)$ uniformly converges to $\ell(\bs \theta):=\ave_{\x_1,\x_2,d_1,d_2}[
%	\mu_{12}(\bs \theta_*)
%	\log
%	\mu_{12}(\bs \theta)
%	-
%	\mu_{12}(\bs \theta)
%]$ in probability over $\bs \Theta$. 
%Since the objective function converges uniformly, we can also prove by~\citet{chernozhukov2007estimation} that 
%$$
%d_{\text{H}}(\hat{\bs \Theta}_n,\bs \Theta_*) \overset{p}{\to} 0,
%$$ 
%where $\hat{\bs \Theta}_n \subset \bs \Theta$ and $\bs \Theta_*:=\{\bs \theta \in \bs \Theta \mid \mu_{12}(\bs \theta_*)=\mu_{12}(\bs \theta) \} \subset \bs \Theta$ are sets of global maximizers of ${\ell}_n(\bs \theta)$ and $\ell(\bs \theta)$, respectively. $d_{\text{H}}(\cdot,\cdot)$ is a Hausdorff distance which shows $\max$-$\min$ $L^2$-distance between two sets. 
%By further assuming that the true parameter $\bs \theta_*$ belongs to the solution space 
%$\bs \Theta$, $\bs \Theta_*$ can be written as $\{\bs \theta \in \bs \Theta \mid \mu_{12}(\bs \theta_*)=\mu_{12}(\bs \theta) \}$. 

%\section{GENERALITY OF PMvGE}
%\label{sec:generality_pmvge}
%In Section~\ref{subsec:pmvge_classifiers} and \ref{subsec:optimization_linear_transformation}, generality of PMvGE is shown, by proving that neural network classifier of multiple classes and CDMCA~\citep{shimodaira2016cross} are included in PMvGE. 
%Since CDMCA generalizes a variety of existing methods, PMvGE at least generalizes them as well. 

\section{Real data analysis}
\label{sec:real_data}

\subsection{Experiments on Citation dataset (1-view)}
\label{subsec:real_1view}

\textbf{Dataset:} 
We use Cora dataset~\citep{sen2008collective} of citation network with  2,708 nodes and 5,278 ordered edges. 
Each node $v_i$ represents a document, which has $1,433$-dimensional (bag-of-words) data vector
$\x_i \in \{0,1\}^{1433}$ and a class label of 7 classes.
Each directed edge represents citation from a document $v_i$ to another document $v_j$.
We set the link weight as $w_{ij}=w_{ji}=1$ by ignoring the direction, and $w_{ij}=0$ otherwise. 
There is no cross or self-citation.
We divide the dataset into 
training set consisting of $2,166$ nodes ($80 \%$) with their edges, and test set consisting of remaining $542$ nodes ($20\%$) with their edges. We utilize $20\%$ of the training set for validation. Hyper-parameters are tuned by utilizing the validation set.

We compare PMvGE with several feature learning methods:
Stochastic Block Model~\citep[SBM]{holland1983stochastic},
ISOMAP~\citep{tenenbaum2000global}, 
Locally Linear Embedding~\citep[LLE]{roweis2000nonlinear}, 
Spectral Graph Embedding~\citep[SGE]{belkin2001laplacian}, 
Multi Dimensional Scaling~\citep[MDS]{kruskal1964multidimensional}, DeepWalk~\citep{perozzi2014deepwalk}, and GraphSAGE~\citep{hamilton2017inductive}.

\textbf{NN for PMvGE:}
2-layer fully-connected network, which consists of 3,000 tanh hidden units and 1,000 tanh output units, is used.
The network is trained by Adam~\citep{kingma2014adam} with batch normalization.
The learning rate is starting from $0.0001$ and attenuated by $1/10$ for every 100 iterations. 
Negative sampling rate $r$ and minibatch size $m$ are set as $1$ and $512$, respectively, and the number of iterations is $200$.

\textbf{Parameter tuning:}
For each method, parameters are tuned on validation sets. Especially, the dimension of feature vectors is selected from $\{50,100,150,200\}$. 

\textbf{Label classification (Task 1):} We classify the documents into 7 classes using logistic regression with the feature vector as input and the class label as output.
We utilize LibLinear~\citep{fan2008liblinear} for the implementation.

\textbf{Clustering (Task 2):} 
The $k$-means clustering of the feature vectors is performed for unsupervised learning of document clusters.
The number of clusters is set as $7$.

\textbf{Results:} 
The quality of classification is evaluated by classification accuracy in Task~1, and
Normalized Mutual Information~(NMI) in Task~2.
Sample averages and standard deviations over 10 experiments are shown in Table~\ref{table:exp_result}. 
In experiment (A), we apply methods to both training set and test set, and evaluate them by test set. 
In (B), we apply methods to only the training set, and evaluate them by test set. 
SGE, MDS, and DeepWalk are not inductive, and they cannot be applied to unseen data vectors in (B). 
PMvGE outperforms the other methods in both experiments.

\subsection{Experiments on AwA dataset (2-view)}
\label{subsec:real_2view}

\textbf{Dataset:} We use Animal with Attributes~(AwA) dataset~\citep{lampert2009learning}
with 30,475 images for view-1 and 85 attributes for view-2. 
We prepared $4,096$ dimensional DeCAF data vector~\citep{donahue2014decaf} for each image,
and $300$ dimensional GloVe~\citep{pennington2014glove} data vector for each attribute.
Each image is associated with some attributes. 
We set $w_{ij}=1$ for the associated pairs between the two views, and $w_{ij}=0$ otherwise.
In addition to the attributes, each image has a class label of 50 classes.
We resampled 50 images from each of 50 classes; in total, 2500 images.
In each experiment, we split the 2500 images into 1500 training images and 100 test images.
A validation set of 300 images is sampled from the training images.

We compare PMvGE with CCA, DCCA~\citep{andrew2013deep}, SGE, DeepWalk, and GraphSAGE.

%$\mathcal{X}^{\text{img}}_{\text{valid}} \subset \mathcal{X}^{\text{img}}_{\text{train}}$.
%$\mathcal{X}^{\text{img}}$ into training image set $\mathcal{X}^{\text{img}}_{\text{train}}$, which consists of 1,500 images, and test image set $\mathcal{X}^{\text{img}}_{\text{test}}$, which consists of remaining 1,000 images. A set of 300 images picked up from the training image set is used as validation set $\mathcal{X}^{\text{img}}_{\text{valid}} \subset \mathcal{X}^{\text{img}}_{\text{train}}$.
%We denote the set of associations across $\mathcal{X}^{\text{img}}_{\text{train}}$ and $\mathcal{X}^{\text{attr}}$ as $\mathcal{W}_{\text{train}}$. $\mathcal{W}_{\text{valid}},\mathcal{W}_{\text{test}}$ are similarly defined. 

\textbf{NN for PMvGE:}
Each view uses a 2-layer fully-connected network, which consists of
2000 tanh hidden units and 100 tanh output units.
The dimension of the feature vector is $K=100$.
Adam is used for optimization with Batch normalization and Dropout ($p = 0.5$).
Minibatch size, learning rate, and momentum are tuned on the validation set.
We monitor the score on the validation set for early stopping.
%validation setでのスコアに基づき，一つの中間層をもつ次の2 layer networkがアーキテクチャとして選択された．中間層は2000 tanh unitsからなり，出力層は100 tanh unitsである．中間層の出力にはBatch normalizationが適応され，各層には$p=0.5$でDropoutが適用された．また，stabilityのためにnetworkの出力は常にL2 normalizeする．提案法の最適化にはAdamを用い，Early stoppingにより学習終了した．

\textbf{Parameter tuning:}
For each method, parameters are tuned on validation sets. Especially, the dimension of feature vectors is selected from $\{10,50,100,150\}$.

\textbf{Link prediction (Task~3):}
For each query image, we rank attributes according to the cosine similarity of feature vectors across views.
An attribute is regarded as correct if it is associated with the query image.

\textbf{Results:} 
The quality of the ranked list of attributes is measured by Average Precision~(AP) score in Task~3.
Sample averages and standard deviations over 10 experiments are shown in Table~\ref{table:exp_result}.
In experiment (A), we apply methods to both training set and test set, and evaluate them by test set.
The whole training set is used for validation. 
In experiment (B), we apply methods to only the training set, and evaluate them by test set.
20$\%$ of training set is used for validation. 
PMvGE outperforms the other methods including DCCA.
While DeepWalk shows good performance in experiment (A), DeepWalk and SGE cannot be applied to unseen data vectors in (B). 
Unlike SGE and DeepWalk which only consider the associations, 1-view feature learning methods such as GraphSAGE cannot be applied to this AwA dataset since the dimension of data vectors is different depending on the view. So we do not perform 1-view methods in Task 3.

%\vspace{-1em}
%
%\begin{table}[htbp]
%\centering
%\caption{Experiments on AwA dataset (2-view). The larger values are better.}
%\label{table:awa}
%	\begin{tabular}{l|cc}
%	& (A) & (B) \\
%	\hline
%	CCA & $45.5 \pm 0.20$ & $42.4 \pm 0.30$ \\
%	DCCA & $41.4 \pm 0.30$ & $41.2 \pm 0.35$ \\
%	SGE & $43.5 \pm 0.39$ & - \\
%	DeepWalk & $71.3 \pm 0.57$ & - \\
%	\textbf{PMvGE} & $\mathbf{71.5 \pm 0.48}$ & $\mathbf{70.5 \pm 0.53}$ \\
%	\end{tabular}
%\end{table}	
%

%
%We conduct two experiments (A) and (B). Especially, experiment (B) tries to predict links between unseen feature vectors, so some approaches including SE and DeepWalk cannot predict the links. 
%In experiment (A), 
%($\mathcal{X}^{\text{img}}_{\text{train}},\mathcal{X}^{\text{img}}_{\text{test}}, \mathcal{X}^{\text{attr}}$) is used for training, ($\mathcal{X}^{\text{img}}_{\text{train}},\mathcal{X}^{\text{attr}}$) for validation, and ($\mathcal{X}^{\text{img}}_{\text{test}},\mathcal{X}^{\text{attr}}$) for evaluation. Associations across feature vectors are also used. 
%In experiment (B), 
%($\mathcal{X}^{\text{img}}_{\text{train}},\mathcal{X}^{\text{attr}}$) is used for training, ($\mathcal{X}^{\text{img}}_{\text{valid}},\mathcal{X}^{\text{attr}}$) for validation, and ($\mathcal{X}^{\text{img}}_{\text{test}},\mathcal{X}^{\text{attr}}$) for evaluation. 
%training setで埋め込みを学習した後，埋め込まれたtest setの画像をクエリとしてcosine similarityによって近い順にattribute nodesを列挙し，Average Precision (AP) を評価する．検索されたattributeがクエリ画像に関連づけられている場合を正解とする．学習した埋め込みによってtest setにおけるAPを導出する手続きを10個のrandom seedのもとで繰り返し，平均値と標準偏差をスコアとして報告する．

\begin{table}[htbp]
  \centering
  \caption{Task 1 and Task 2 for the experiment on Citation dataset ($D=1$), and Task 3 for the experiment on AwA dataset ($D=2$).
  The larger values are better.}
  \label{table:exp_result}
  \scalebox{0.9}{
\begin{tabular}{cl|cc}
& & (A) & (B) \\
\hline
\multirow{7}{*}{\shortstack[l]{Task 1 \\ ($D=1$)}}
%& SBM & - & - \\
& ISOMAP & $54.5 \pm 1.78$ & $54.8 \pm 2.43$ \\
& LLE & $30.2 \pm 1.91$ & $31.9 \pm 2.62$ \\
& SGE & $47.6 \pm 1.64$ & - \\
& MDS & $29.8 \pm 2.25$ & - \\
& DeepWalk & $54.2 \pm 2.04$ & - \\
& GraphSAGE & $60.8 \pm 1.73$ & $57.1 \pm 1.61$ \\
& \textbf{PMvGE} & $\mathbf{74.8 \pm 2.55}$ & $\mathbf{71.1 \pm 2.10}$ \\
\hline
\multirow{8}{*}{\shortstack[l]{Task 2 \\ ($D=1$)}}
& SBM & $4.37 \pm 1.44$ & $2.81 \pm 0.10$ \\
& ISOMAP & $13.0 \pm 0.36$ & $14.3 \pm 1.98$ \\
& LLE & $7.40 \pm 3.40$ & $9.47 \pm 3.00$ \\
& SGE & $1.41 \pm 0.34$ & - \\
& MDS & $2.81 \pm 0.10$ & - \\
& DeepWalk & $16.7 \pm 1.05$ & - \\
& GraphSAGE & $19.6 \pm 0.93$ & $12.4 \pm 3.00$ \\
& \textbf{PMvGE} & $\mathbf{35.9 \pm 0.88}$ & $\mathbf{30.5 \pm 3.90}$ \\
\hline
\multirow{5}{*}{\shortstack[l]{Task 3 \\ ($D=2$)}}
& CCA & $45.5 \pm 0.20$ & $42.4 \pm 0.30$ \\
& DCCA & $41.4 \pm 0.30$ & $41.2 \pm 0.35$ \\
& SGE & $43.5 \pm 0.39$ & - \\
& DeepWalk & $71.3 \pm 0.57$ & - \\
& \textbf{PMvGE} & $\mathbf{71.5 \pm 0.48}$ & $\mathbf{70.5 \pm 0.53}$ \\
\end{tabular}
}
\end{table}

\textbf{Locality of each-view is preserved through neural networks:} 
To see whether the locality of input is preserved through neural networks in PMvGE, we computed the Spearman's rank correlation coefficient between $\langle \bs x, \bs x'\rangle$ and
 $\langle f_{\bs \psi}^{(d)}(\bs x), f_{\bs \psi}^{(d)}(\bs x') \rangle$ for view-$d$ data vectors $\bs x, \bs x'$ in AwA dataset ($d=1,2$).
For DeCAF (view-1) and GloVe (view-2) inputs, the values are $0.722 \pm 0.058$ and $0.811 \pm 0.082$, respectively. This result indicates that the feature vectors of PMvGE preserves the similarities of the data vectors fairly well.

\section{Conclusion}
\label{sec:conclusion}
We presented a simple probabilistic framework for multi-view learning with many-to-many associations. 
We name the framework as Probabilistic Multi-view Graph Embedding~(PMvGE). 
Various existing methods are approximately included in PMvGE.
%, and the model can be extended to non-linear settings by incorporating neural networks. 
We gave theoretical justification and practical estimation algorithm to PMvGE.
Experiments on real-world datasets showed that PMvGE outperforms existing methods.

%\bibliographystyle{apalike}
%\bibliography{./bib_pmvge}

\begin{thebibliography}{}

\bibitem[Andrew et~al., 2013]{andrew2013deep}
Andrew, G., Arora, R., Bilmes, J., and Livescu, K. (2013).
\newblock Deep {C}anonical {C}orrelation {A}nalysis.
\newblock In {\em Proceedings of the International Conference on Machine
  Learning (ICML)}, pages 1247--1255.

\bibitem[Belkin and Niyogi, 2001]{belkin2001laplacian}
Belkin, M. and Niyogi, P. (2001).
\newblock Laplacian {E}igenmaps and {S}pectral {T}echniques for {E}mbedding and
  {C}lustering.
\newblock In {\em Advances in Neural Information Processing Systems (NIPS)},
  volume~14, pages 585--591.

\bibitem[Bishop, 2006]{prml}
Bishop, C.~M. (2006).
\newblock {\em Pattern Recognition and Machine Learning}.
\newblock Springer.

\bibitem[Chung, 1997]{chung1997spectral}
Chung, F.~R. (1997).
\newblock {\em Spectral {G}raph {T}heory}.
\newblock American Mathematical Society.

\bibitem[Courant and Hilbert, 1989]{courant89}
Courant, R. and Hilbert, D. (1989).
\newblock {\em Methods of Mathematical Physics}, volume~1.
\newblock Wiley, New York.

\bibitem[Cybenko, 1989]{cybenko1989approximation}
Cybenko, G. (1989).
\newblock Approximation by {S}uperpositions of a {S}igmoidal {F}unction.
\newblock {\em Mathematics of Control, Signals, and Systems (MCSS)},
  2(4):303--314.

\bibitem[Dai et~al., 2018]{dai2018adversarial}
Dai, Q., Li, Q., Tang, J., and Wang, D. (2018).
\newblock Adversarial {N}etwork {E}mbedding.
\newblock In {\em Proceedings of the conference on Artificial Intelligence
  (AAAI)}.

\bibitem[Defferrard et~al., 2016]{defferrard2016convolutional}
Defferrard, M., Bresson, X., and Vandergheynst, P. (2016).
\newblock Convolutional {N}eural {N}etworks on {G}raphs with {F}ast {L}ocalized
  {S}pectral {F}iltering.
\newblock In {\em Advances in Neural Information Processing Systems (NIPS)},
  pages 3844--3852.

\bibitem[Donahue et~al., 2014]{donahue2014decaf}
Donahue, J., Jia, Y., Vinyals, O., Hoffman, J., Zhang, N., Tzeng, E., and
  Darrell, T. (2014).
\newblock De{CAF}: A {D}eep {C}onvolutional {A}ctivation {F}eature for
  {G}eneric {V}isual {R}ecognition.
\newblock In {\em Proceedings of the International Conference on Machine
  Learning (ICML)}, pages 647--655.

\bibitem[Fan et~al., 2008]{fan2008liblinear}
Fan, R.-E., Chang, K.-W., Hsieh, C.-J., Wang, X.-R., and Lin, C.-J. (2008).
\newblock {LIBLINEAR}: {A} library for large linear classification.
\newblock {\em Journal of Machine Learning Research (JMLR)}, 9:1871--1874.

\bibitem[Fukui et~al., 2016]{fukui2016image}
Fukui, K., Okuno, A., and Shimodaira, H. (2016).
\newblock Image and tag retrieval by leveraging image-group links with
  multi-domain graph embedding.
\newblock In {\em Proceedings of the IEEE International Conference on Image
  Processing (ICIP)}, pages 221--225.

\bibitem[Funahashi, 1989]{funahashi1989approximate}
Funahashi, K.-I. (1989).
\newblock On the approximate realization of continuous mappings by neural
  networks.
\newblock {\em Neural Networks}, 2(3):183--192.

\bibitem[Goodfellow et~al., 2016]{Goodfellow-et-al-2016-Book}
Goodfellow, I., Bengio, Y., and Courville, A. (2016).
\newblock {\em Deep {L}earning}.
\newblock MIT Press.

\bibitem[Grover and Leskovec, 2016]{grover2016node2vec}
Grover, A. and Leskovec, J. (2016).
\newblock node2vec: {S}calable {F}eature {L}earning for {N}etworks.
\newblock In {\em Proceedings of the ACM International Conference on Knowledge
  Discovery and Data mining (SIGKDD)}, pages 855--864. ACM.

\bibitem[Hamilton et~al., 2017]{hamilton2017inductive}
Hamilton, W.~L., Ying, Z., and Leskovec, J. (2017).
\newblock Inductive {R}epresentation {L}earning on {L}arge {G}raphs.
\newblock {\em Advances in Neural Information Processing Systems (NIPS)}, pages
  1025--1035.

\bibitem[He and Niyogi, 2004]{he2003locality}
He, X. and Niyogi, P. (2004).
\newblock Locality {P}reserving {P}rojections.
\newblock In {\em Advances in Neural Information Processing Systems (NIPS)},
  pages 153--160.

\bibitem[Holland et~al., 1983]{holland1983stochastic}
Holland, P.~W., Laskey, K.~B., and Leinhardt, S. (1983).
\newblock Stochastic blockmodels: First steps.
\newblock {\em Social Networks}, 5(2):109--137.

\bibitem[Hotelling, 1936]{hotelling1936relations}
Hotelling, H. (1936).
\newblock Relations between two sets of variates.
\newblock {\em Biometrika}, 28(3/4):321--377.

\bibitem[Huang et~al., 2012]{huang2012cross}
Huang, Z., Shan, S., Zhang, H., Lao, S., and Chen, X. (2012).
\newblock Cross-view {G}raph {E}mbedding.
\newblock In {\em Proceedings of the Asian Conference on Computer Vision
  (ACCV)}, pages 770--781.

\bibitem[Kettenring, 1971]{kettenring1971canonical}
Kettenring, J.~R. (1971).
\newblock Canonical {A}nalysis of {S}everal {S}ets of {V}ariables.
\newblock {\em Biometrika}, 58(3):433--451.

\bibitem[Kingma and Ba, 2015]{kingma2014adam}
Kingma, D. and Ba, J. (2015).
\newblock Adam: A {M}ethod for {S}tochastic {O}ptimization.
\newblock In {\em Proceedings of the International Conference on Learning
  Representations (ICLR)}.

\bibitem[Kipf and Welling, 2017]{kipf2017semi}
Kipf, T.~N. and Welling, M. (2017).
\newblock Semi-supervised classification with graph convolutional networks.
\newblock In {\em Proceedings of the International Conference on Learning
  Representations (ICLR)}.

\bibitem[Kruskal, 1964]{kruskal1964multidimensional}
Kruskal, J.~B. (1964).
\newblock Multidimensional scaling by optimizing goodness of fit to a nonmetric
  hypothesis.
\newblock {\em Psychometrika}, 29(1):1--27.

\bibitem[Lai and Fyfe, 2000]{lai2000kernel}
Lai, P.~L. and Fyfe, C. (2000).
\newblock Kernel and nonlinear canonical correlation analysis.
\newblock {\em International Journal of Neural Systems}, 10(05):365--377.

\bibitem[Lampert et~al., 2009]{lampert2009learning}
Lampert, C.~H., Nickisch, H., and Harmeling, S. (2009).
\newblock Learning {T}o {D}etect {U}nseen {O}bject {C}lasses by
  {B}etween-{C}lass {A}ttribute {T}ransfer.
\newblock In {\em Proceedings of the IEEE International Conference on Computer
  Vision and Pattern Recognition (CVPR)}, pages 951--958. IEEE.

\bibitem[Mercer, 1909]{mercer1909functions}
Mercer, J. (1909).
\newblock Functions of positive and negative type, and their connection with
  the theory of integral equations.
\newblock {\em Philosophical Transactions of the Royal Society, London. Series
  A.}, 209:415--446.

\bibitem[Mikolov et~al., 2013]{mikolov2013distributed}
Mikolov, T., Sutskever, I., Chen, K., Corrado, G.~S., and Dean, J. (2013).
\newblock Distributed {R}epresentations of {W}ords and {P}hrases and their
  {C}ompositionality.
\newblock In {\em Advances in Neural Information Processing Systems (NIPS)},
  pages 3111--3119.

\bibitem[Nori et~al., 2012]{nori2012multinomial}
Nori, N., Bollegala, D., and Kashima, H. (2012).
\newblock Multinomial {R}elation {P}rediction in {S}ocial {D}ata: {A}
  {D}imension {R}eduction {A}pproach.
\newblock In {\em Proceedings of the AAAI conference on Artificial
  Intelligence}, volume~12, pages 115--121.

\bibitem[Nowicki and Snijders, 2001]{nowicki2001estimation}
Nowicki, K. and Snijders, T. A.~B. (2001).
\newblock Estimation and {P}rediction for {S}tochastic {B}lockstructures.
\newblock {\em Journal of the American Statistical Association},
  96(455):1077--1087.

\bibitem[Oshikiri et~al., 2016]{oshikiri2016cross}
Oshikiri, T., Fukui, K., and Shimodaira, H. (2016).
\newblock Cross-{L}ingual {W}ord {R}epresentations via {S}pectral {G}raph
  {E}mbeddings.
\newblock In {\em Proceedings of the Annual Meeting of the Association for
  Computational Linguistics (ACL)}, volume~2, pages 493--498.

\bibitem[Pennington et~al., 2014]{pennington2014glove}
Pennington, J., Socher, R., and Manning, C. (2014).
\newblock Glo{V}e: {G}lobal {V}ectors for {W}ord {R}epresentation.
\newblock In {\em Proceedings of the Conference on Empirical Methods in Natural
  Language Processing (EMNLP)}, pages 1532--1543.

\bibitem[Perozzi et~al., 2014]{perozzi2014deepwalk}
Perozzi, B., Al-Rfou, R., and Skiena, S. (2014).
\newblock Deep{W}alk: {O}nline {L}earning of {S}ocial {R}epresentations.
\newblock In {\em Proceedings of the ACM International Conference on Knowledge
  Discovery and Data mining (SIGKDD)}, pages 701--710. ACM.

\bibitem[Rendle, 2010]{rendle2010factorization}
Rendle, S. (2010).
\newblock Factorization {M}achines.
\newblock In {\em Proceedings of the IEEE International Conference on Data
  Mining (ICDM)}, pages 995--1000. IEEE.

\bibitem[Roweis and Saul, 2000]{roweis2000nonlinear}
Roweis, S.~T. and Saul, L.~K. (2000).
\newblock Nonlinear {D}imensionality {R}eduction by {L}ocally {L}inear
  {E}mbedding.
\newblock {\em Science}, 290(5500):2323--2326.

\bibitem[Scarselli et~al., 2009]{scarselli2009graph}
Scarselli, F., Gori, M., Tsoi, A.~C., Hagenbuchner, M., and Monfardini, G.
  (2009).
\newblock The {G}raph {N}eural {N}etwork {M}odel.
\newblock {\em IEEE Transactions on Neural Networks}, 20(1):61--80.

\bibitem[Sen et~al., 2008]{sen2008collective}
Sen, P., Namata, G., Bilgic, M., Getoor, L., Galligher, B., and Eliassi-Rad, T.
  (2008).
\newblock Collective {C}lassification in {N}etwork {D}ata.
\newblock {\em AI magazine}, 29(3):93.

\bibitem[Shimodaira, 2016]{shimodaira2016cross}
Shimodaira, H. (2016).
\newblock Cross-validation of matching correlation analysis by resampling
  matching weights.
\newblock {\em Neural Networks}, 75:126--140.

\bibitem[Sun, 2013]{sun2013survey}
Sun, S. (2013).
\newblock A survey of multi-view machine learning.
\newblock {\em Neural Computing and Applications}, 23(7-8):2031--2038.

\bibitem[Tang et~al., 2015]{tang2015line}
Tang, J., Qu, M., Wang, M., Zhang, M., Yan, J., and Mei, Q. (2015).
\newblock {LINE}: {L}arge-scale {I}nformation {N}etwork {E}mbedding.
\newblock In {\em Proceedings of the International Conference on World Wide Web
  (WWW)}, pages 1067--1077. International World Wide Web Conferences Steering
  Committee.

\bibitem[Telgarsky, 2017]{pmlr-v70-telgarsky17a}
Telgarsky, M. (2017).
\newblock Neural networks and rational functions.
\newblock In Precup, D. and Teh, Y.~W., editors, {\em Proceedings of the
  International Conference on Machine Learning (ICML)}.

\bibitem[Tenenbaum et~al., 2000]{tenenbaum2000global}
Tenenbaum, J.~B., De~Silva, V., and Langford, J.~C. (2000).
\newblock A {G}lobal {G}eometric {F}ramework for {N}onlinear {D}imensionality
  {R}eduction.
\newblock {\em Science}, 290(5500):2319--2323.

\bibitem[Wang et~al., 2016]{wang2016deep}
Wang, W., Yan, X., Lee, H., and Livescu, K. (2016).
\newblock Deep {V}ariational {C}anonical {C}orrelation {A}nalysis.
\newblock {\em arXiv preprint arXiv:1610.03454}.

\bibitem[Yan and Mikolajczyk, 2015]{yan2015deep}
Yan, F. and Mikolajczyk, K. (2015).
\newblock Deep {C}orrelation for {M}atching {I}mages and {T}ext.
\newblock In {\em Proceedings of the IEEE Conference on Computer Vision and
  Pattern Recognition (CVPR)}, pages 3441--3450. IEEE.

\bibitem[Yan et~al., 2007]{yan2007graph}
Yan, S., Xu, D., Zhang, B., Zhang, H.-J., Yang, Q., and Lin, S. (2007).
\newblock Graph {E}mbedding and {E}xtensions: {A} {G}eneral {F}ramework for
  {D}imensionality {R}eduction.
\newblock {\em IEEE transactions on Pattern Analysis and Machine Intelligence
  (PAMI)}, 29(1):40--51.

\bibitem[Yarotsky, 2016]{yarotsky2016error}
Yarotsky, D. (2016).
\newblock Error bounds for approximations with deep {ReLU} networks.
\newblock {\em arXiv preprint arXiv:1610.01145}.

\bibitem[Zhao et~al., 2017]{zhao2017multi}
Zhao, J., Xie, X., Xu, X., and Sun, S. (2017).
\newblock Multi-view learning overview: Recent progress and new challenges.
\newblock {\em Information Fusion}, 38:43--54.

\bibitem[Zheng et~al., 2006]{zheng2006facial}
Zheng, W., Zhou, X., Zou, C., and Zhao, L. (2006).
\newblock Facial expression recognition using kernel canonical correlation
  analysis ({KCCA}).
\newblock {\em IEEE transactions on Neural Networks}, 17(1):233--238.

\end{thebibliography}

\begin{thebibliography}{}

\bibitem[Chernozhukov et~al., 2007]{chernozhukov2007estimation}
Chernozhukov, V., Hong, H., and Tamer, E. (2007).
\newblock Estimation and {C}onfidence {R}egions for {P}arameter {S}ets in
  {E}conometric {M}odels.
\newblock {\em Econometrica}, 75(5):1243--1284.

\bibitem[Cybenko, 1989]{cybenko1989approximation}
Cybenko, G. (1989).
\newblock Approximation by {S}uperpositions of a {S}igmoidal {F}unction.
\newblock {\em Mathematics of Control, Signals, and Systems (MCSS)},
  2(4):303--314.

\bibitem[Minh et~al., 2006]{minh2006mercer}
Minh, H.~Q., Niyogi, P., and Yao, Y. (2006).
\newblock Mercer's {T}heorem, {F}eature {M}aps, and {S}moothing.
\newblock In {\em International Conference on Computational Learning Theory},
  pages 154--168. Springer.

\bibitem[Newey, 1991]{newey1991uniform}
Newey, W.~K. (1991).
\newblock Uniform {C}onvergence in {P}robability and {S}tochastic
  {E}quicontinuity.
\newblock {\em Econometrica}, pages 1161--1167.

\bibitem[Shimodaira, 2016]{shimodaira2016cross}
Shimodaira, H. (2016).
\newblock Cross-validation of matching correlation analysis by resampling
  matching weights.
\newblock {\em Neural Networks}, 75:126--140.

\bibitem[Telgarsky, 2017]{pmlr-v70-telgarsky17a}
Telgarsky, M. (2017).
\newblock Neural networks and rational functions.
\newblock In Precup, D. and Teh, Y.~W., editors, {\em Proceedings of the
  International Conference on Machine Learning (ICML)}.

\end{thebibliography}

%\input{pmvge.bbl}

\clearpage
\appendix
\onecolumn

\begin{flushleft}
\textbf{\Large Supplementary Material:} 
{\Large A probabilistic framework for multi-view feature learning with many-to-many associations via neural networks}
\end{flushleft}
\hrulefill

\section{Proof of Theorem~\ref{theo:universal_approximate}}
\label{sec:univ_supp}

Since $g_*:[-M',M']^{2K_*} \to \mathbb{R}$ is a positive definite kernel on a compact set,
it follows from Mercer's theorem that
there exist positive eigenvalues $\{\lambda_k\}_{k=1}^{\infty}$ and
continuous eigenfunctions $\{\phi_k\}_{k=1}^{\infty}$ such that
\[
g_*(\bs y_*,\bs y_*') = \sum_{k=1}^{\infty} \lambda_k \phi_k(\bs y_*)\phi_k(\bs y_*'),\quad
\bs y_*, \bs y_*' \in [-M',M']^{K_*},
\]
where the convergence is absolute and uniform~\citep{minh2006mercer}.
The uniform convergence implies that for any $\varepsilon_1>0$ there exists $K_0 \in \mathbb{N}$
such that 
\begin{align*}
\sup_{(\bs y_*,\bs y_*') \in [-M',M']^{2K_*}}
	\bigg|
		g_*(\bs y_*,\bs y_*')
		-
		\sum_{k=1}^{K} \lambda_k \phi_k(\bs y_*) \phi_k(\bs y_*')
	\bigg|
	<
	\varepsilon_1,
	\quad
	K \geq K_0.
\end{align*}
This means $g_*(\bs y_*,\bs y'_*)  \approx \langle \bs \Phi_K(\bs y_*),\bs \Phi_K(\bs y_*')\rangle$
for a feature map $\bs \Phi_K(\bs y_*) = (\sqrt{\lambda_k} \phi_k(\bs y_*) )_{k=1}^K$.

We fix $K$ and consider approximation of $h_k^{(d)}(\bs x) := \sqrt{\lambda_k} \phi_k( f_*^{(d)}(\bs x))$ below.
Since $h_k^{(d)}$ are continuous functions on a compact set, there exists $C=C(K)>0$ such that
\[
	\sup_{\bs x \in [-M,M]^{p_d}} | h_k^{(d)}(\bs x) | < C, \quad k=1,\ldots, K,\, d=1,\ldots, D.
\]
Let us write the neural networks as $f_{\bs \psi}^{(d)} = (f_1^{(d)}, \ldots, f_K^{(d)})$, where
$f_k^{(d)}: \mathbb{R}^{p_d} \to \mathbb{R}$, $d=1,\ldots, D$, $k=1,\ldots,K$, are
two-layer neural networks with $T$ hidden units.
Since $h_k^{(d)}$ are continuous functions,
it follows from the universal approximation theorem \citep{cybenko1989approximation,pmlr-v70-telgarsky17a} that
for any $\varepsilon_2>0$, there exists $T_0(K) \in \mathbb{N}$ such that
\[
	\sup_{\bs x \in [-M,M]^{p_d}} | h_k^{(d)}(\bs x)  - f_k^{(d)}(\bs x)  | < \varepsilon_2, \quad k=1,\ldots, K,\, d=1,\ldots, D
\]
for $T\ge T_0(K)$.
Therefore, for all $d,e \in \{1,2,\ldots,D\}$, we have
\begin{align*}
\sup_{(\bs x,\bs x') \in [-M,M]^{p_d+p_e}}
&\biggl|
g_*\left(
	f^{(d)}_*(\bs x)
	,
	f^{(e)}_*(\bs x')
\right)
-
\sum_{k=1}^K f_k^{(d)}(\bs x) f_k^{(e)}(\bs x') \biggr|\\
%%%%%%%%%%%%%%%%%%%%%%%%%%%%%%%%%%%%%%%%%%%%%%%%%%%%%%%%%%%
&\leq
\sup_{(\bs x,\bs x') \in [-M,M]^{p_d+p_e}}
\biggl|
g_*\left(
	f^{(d)}_*(\bs x)
	,
	f^{(e)}_*(\bs x')
\right)
-
\sum_{k=1}^K h_k^{(d)}(\bs x) h_k^{(e)}(\bs x') \biggr| \\
&\hspace{2em}
+
\sup_{(\bs x,\bs x') \in [-M,M]^{p_d+p_e}}
\biggl|
\sum_{k=1}^K h_k^{(d)}(\bs x)  \left( h_k^{(e)}(\bs x') -  f_k^{(e)}(\bs x')  \right)
\biggr| \\
&\hspace{2em}
+
\sup_{(\bs x,\bs x') \in [-M,M]^{p_d+p_e}}
\biggl|
\sum_{k=1}^K  \left( h_k^{(d)}(\bs x) - f_k^{(d)}(\bs x)  \right) f_k^{(e)}(\bs x') 
\biggr| \\ %% 次の不等式
%%%%%%%%%%%%%%%%%%%%%%%%%%%%%%%%%%%%%%%%%%%%%%%%%%%%%%%%%%%
&\leq 
\sup_{\bs y_*,\bs y_*' \in [-M',M']^{K_*}}
\biggl|
g_*\left(
	\bs y_*,\bs y_*'
\right)
-
\sum_{k=1}^{K}
\lambda_k 
\phi_k(\bs y_*)
\phi_k(\bs y_*')
\biggr| \\
&\hspace{2em}
+
\sum_{k=1}^K
\sup_{\bs x \in [-M,M]^{p_d}}
|h_k^{(d)}(\bs x) |
\sup_{\bs x' \in [-M,M]^{p_e}}
 \left| h_k^{(e)}(\bs x') -  f_k^{(e)}(\bs x')  \right| \\
&\hspace{2em}
+
\sum_{k=1}^K
\sup_{\bs x \in [-M,M]^{p_d}}
\left| h_k^{(d)}(\bs x) - f_k^{(d)}(\bs x)  \right|
\sup_{\bs x' \in [-M,M]^{p_e}}
|f_k^{(e)}(\bs x') | \\
%%%%%%%%%%%%%%%%%%%%%%%%%%%%%%%%%%%%%%%%%%%%%%%%%%%%%%%%%%%
& <
\varepsilon_1 + K C \varepsilon_2 + K  \varepsilon_2 (C+\varepsilon_2).
\end{align*}
By letting
$\varepsilon_1=\varepsilon/2,
\varepsilon_2= \min\left(C,  {\varepsilon}/({6KC}) \right)$,
the last formula becomes smaller than $\varepsilon$, thus proving
\[
\sup_{(\bs x,\bs x') \in [-M,M]^{p_d+p_e}}
\biggl|
g_*\left(
	f^{(d)}_*(\bs x)
	,
	f^{(e)}_*(\bs x')
\right)
-
\sum_{k=1}^K f_k^{(d)}(\bs x) f_k^{(e)}(\bs x') 
\biggr|
<
\varepsilon,\quad d,e=1,\ldots,D.
\]

\qed

\section{Consistency of MLE in PMvGE} 
\label{sec:consistency_supp}

In this section, we provide technical details of the argument of Section~\ref{subsec:consistency_pmvge}.

%\subsection{Generative model}
For proving the consistency of MLE, we introduce the following generative model.
Let $d_i$, $i=1,\ldots, n$, be random variables independently distributed with the probability $\mathbb{P}(d_i=d)=\eta^{(d)} \in (0,1)$ where $\sum_{d=1}^{D}\eta^{(d)}=1$. 
Data vectors are also treated as random variables. The conditional distribution of $\bs x_i$ given $d_i$ is
$$
\x_i \mid d_i
\overset{\text{indep.}}{\sim}
q^{(d_i)},\quad i=1,\ldots,n,
$$ 
where $q^{(d_i)}$ is a distribution on a compact support in $\mathbb{R}^{p_{d_i}}$.
Let us denote eq.~(\ref{eq:mu_ij}) as $\mu_{ij}(\bs x_i, \bs x_j, d_i, d_j | \bs \alpha,\bs \psi)$ 
for indicating the dependency on $(\bs x_i, \bs x_j, d_i, d_j)$.
The conditional distributions of link weights are already specified in (\ref{eq:w_model}) as

%ノード$i,j$間の関連の強さが平均$\mu_{ij}$のポアソン分布
%\end{itemize}
%とする．制約条件$\bs \alpha=\bs \alpha^{\top}$より，実質的に推定するパラメータは$\bs \theta:=(\text{vech}\bs \alpha,\bs \psi)$であり，損失関数は$\hat{\ell}_n(\bs \theta)$と書ける．$\bs \theta$の解空間は，$\delta \in (0,1)$を十分小さな定数，$\bs \Psi\subset \mathbb{R}^{q}$をコンパクト集合として
$$
w_{ij} \mid\x_i,\x_j,d_i,d_j \overset{\text{indep.}}{\sim} \text{Po}(\mu^*_{ij}),\quad
i,j=1,\ldots,n,
$$
where $\mu^*_{ij}=\mu_{ij}(\bs x_i, \bs x_j, d_i, d_j | \bs \alpha_*,\bs \psi_*)$ with a true parameter $(\bs \alpha_*,\bs \psi_*)$. 
%\subsection{Our results}
Due to the constraints $\bs \alpha=\bs \alpha^{\top}$ and $w_{ij}=0 \: ((d_i,d_j) \notin \mathcal{D})$, 
the vector of free parameters in $\bs \alpha$ is
$\bs \alpha_{\mathcal{D}} :=\{\alpha^{(d,e)}\}_{(d,e) \in \mathcal{D},d \leq e} \in \mathbb{R}_{\geq 0}^{|\mathcal{D}|}$, and we write $\bs \theta:=(\bs \alpha_{\mathcal{D}},\bs \psi)$. 
Let $\tilde{n}:=|\mathcal{I}_n|=O(n^2)$ denote the number of terms in the sum of $\ell_n(\bs \theta)$ in eq.~(\ref{eq:log_likelihood}).
Then the expected value of $\tilde{n}^{-1}\ell_n(\bs \theta)$ under the generative model with the true parameter $\bs \theta_*$ is expressed as
\[
	\ell(\bs \theta):=\ave_{\x_1,\x_2,d_1,d_2}\Bigl[
	\mu_{12}(\x_1,\x_2,d_1,d_2 | \bs \theta_*)
	\log
	\mu_{12}(\x_1,\x_2,d_1,d_2 | \bs \theta)
	-
	\mu_{12}(\x_1,\x_2,d_1,d_2 | \bs \theta) \Bigr].
\]
If it were the case of i.i.d.~observations of sample size $n$, we would have that $\tilde{n}^{-1}\ell_n(\bs \theta)$ converges to $\ell(\bs \theta)$ as $n\to\infty$ from the law of large numbers.
In Theorem~\ref{theo:uniform_convergence}, we actually prove the uniform convergence in probability,
but we have to pay careful attention to the fact that $n^2$ observations of $(\bs x_i, \bs x_j, d_i, d_j)$ are not independent when indices overlap. 
\begin{theo}
\label{theo:uniform_convergence}
\normalfont
Let us assume that the parameter space of $\bs\theta$ is
$\bs \Theta:=[\delta,1/\delta]^{|\mathcal{D}|}\times\bs \Psi$,
where $\delta \in (0,1)$ is a sufficiently small constant and $\bs \Psi \subset \mathbb{R}^q$ is a compact set.
Assume also that
the transformations $f^{(d)}_{\bs \psi}(\bs x)$, $d=1,\ldots, D$, are Lipschitz continuous with respect to $(\bs \psi,\bs x)$.
Then we have, as $n\to\infty$,
\begin{equation} \label{eq:uniform_convergence_of_lik}
	\sup_{\bs \theta \in \bs \Theta} | \tilde{n}^{-1}\ell_n(\bs \theta) - \ell(\bs \theta) | \overset{p}{\to}0.
\end{equation}
\end{theo}

\begin{proof}
\normalfont
We refer to Corollary 2.2 in \citet{newey1991uniform}. 
%The assertion of Theorem~\ref{theo:uniform_convergence} immediately follows from Corollary 2.2 in \citet{newey1991uniform}. 
This corollary shows 
$\sup_{\bs \theta \in \bs \Theta}|\hat Q_n(\bs \theta)- \bar Q_n(\bs \theta)| = o_p(1)$ under general setting
of $\hat Q_n(\bs \theta)$ and $\bar Q_n(\bs \theta)$.
Here we consider the case
of
$\hat Q_n(\bs \theta) = \tilde{n}^{-1}\ell_n(\bs \theta)$ and
$\bar Q_n(\bs \theta) = \ell(\bs \theta)$.
For showing (\ref{eq:uniform_convergence_of_lik}),
the four conditions of the corollary are written as follows.
(i)~$\bs \Theta$ is compact, 
(ii)~$\tilde{n}^{-1}\ell_n(\bs \theta) - \ell(\bs \theta)\overset{p}{\to} 0$ for each $\bs \theta \in \bs \Theta$, 
(iii)~$\ell(\bs \theta)$ is continuous, 
(iv)~there exists $B_n=O_p(1)$ such that $|\tilde{n}^{-1}\ell_n(\bs \theta)-\tilde{n}^{-1}\ell_n(\bs \theta')| \leq B_n\|\bs \theta-\bs \theta'\|_2$ for all $\bs \theta ,\bs \theta' \in \bs \Theta$. 
The two conditions (i) and (iii) hold obviously,
and thus we verify (ii) and (iv) below.

%\begin{enumerate}[{(1)}]
%\item[(i)] 
Before verifying (ii) and (iv), we first consider an array $\bs Z:=(Z_{ij})$ of random variables
 $Z_{ij} \in \mathcal{Z}$,  $(i,j) \in \mathcal{I}_n$,
 and a bounded and continuous function $h:\mathcal{Z} \to \mathbb{R}$.
We assume that $\mathcal{Z} \subset \mathbb{R}$ is a compact set, and
$Z_{ij}$ is independent of $Z_{kl}$ if $k,l \in \mathcal{R}_n(i,j):=\{(k,l) \in \mathcal{I}_n \mid k,l \in \{1,\ldots, n\} \setminus \{i,j\} \}$ for all $(i,j) \in \mathcal{I}_n$.
Then we have
%$V[\tilde{n}^{-1}\sum_{(i,j) \in \mathcal{I}_n(\mathcal{D})}Z_{ij}]=O(1/n)$ where $\tilde{n}:=|\mathcal{I}_n(\mathcal{D})|$.
%\end{lem}
%\begin{proof}[Proof of Lemma~\ref{lem:doube_array_LLN}]
\begin{align*}
V_{\bs Z}\left[ \frac{1}{\tilde{n}}\sum_{(i,j) \in \mathcal{I}_n}h(Z_{ij}) \right]
&=
\ave_{\bs Z}\left[
	\left(\frac{1}{\tilde{n}}\sum_{(i,j) \in \mathcal{I}_n}h(Z_{ij})\right)^2
\right]
-
\ave_{\bs Z}\left[
	\frac{1}{\tilde{n}}\sum_{(i,j) \in \mathcal{I}_n}h(Z_{ij})
\right]^2 \\
&=
\frac{1}{\tilde{n}^2}
\left\{
\sum_{(i,j) \in \mathcal{I}_n}
\sum_{(k,l) \in \mathcal{I}_n}
\ave_{\bs Z}\left[
	h(Z_{ij})h(Z_{kl})
\right]
-
\left(
\sum_{(i,j) \in \mathcal{I}_n}
\ave_{\bs Z}\left[h(Z_{ij})\right]
\right)
^2
\right\} \\
&=
\frac{1}{\tilde{n}^2}
\sum_{(i,j) \in \mathcal{I}_n}
\sum_{(k,l) \in \mathcal{I}_n \setminus \mathcal{R}_n(i,j)}
\left(
	\ave_{\bs Z}[h(Z_{ij})h(Z_{kl})]
	-
	\ave_{\bs Z}[h(Z_{ij})]\ave_{\bs Z}[h(Z_{kl})]
\right). 
\end{align*}
By considering $|\mathcal{I}_n \setminus \mathcal{R}_n(i,j)|=O(n)$, 
the last formula is $O(\tilde{n}^{-2} \cdot \tilde{n} \cdot n)=O(n^{-1})$. Therefore, 
\begin{align}
V_{\bs Z}\left[ \frac{1}{\tilde{n}}\sum_{(i,j) \in \mathcal{I}_n}h(Z_{ij}) \right]
=
O(n^{-1}).
\label{eq:vz}
\end{align}

Next we evaluate the variance of $\tilde{n}^{-1}\ell_n(\bs \theta)$ to show (ii).
Denoting $\bs W:=(w_{ij}),\bs X:=(\bs x_i),\bs d:=(d_i)$,
%各点での確率収束$\hat{\ell}_n(\bs \theta) \overset{\text{p}}{\to} \ell(\bs \theta), \: \forall \bs \theta \in \bs \Theta$を示す．
\begin{align*}
V_{\bs W,\bs X,\bs d}[\tilde{n}^{-1}\ell_n(\bs \theta)]
&=
\ave_{\bs X,\bs d}[V_{\bs W}[\tilde{n}^{-1}\ell_n(\bs \theta) \mid \bs X,\bs d]]
+
V_{\bs X,\bs d}[\ave_{\bs W}[\tilde{n}^{-1}\ell_n(\bs \theta) \mid \bs X,\bs d]] \\
&=
\ave_{\bs X,\bs d}\left[
	\frac{1}{\tilde{n}^2}
	\sum_{(i,j) \in \mathcal{I}_n}
	\mu_{ij}(\bs \theta_*)
	(\log \mu_{ij}(\bs \theta))^2
\right]
+
V_{\bs X,\bs d}\left[
	\frac{1}{\tilde{n}}
		\sum_{(i,j) \in \mathcal{I}_n}
		(\mu_{ij}(\bs \theta_*)\log \mu_{ij}(\bs \theta)-\mu_{ij}(\bs \theta))
\right],
\end{align*}
for every $\bs \theta \in \bs \Theta$. 
The first term in the last formula is $O(\tilde{n}^{-1} \cdot \tilde{n})=O(\tilde{n}^{-1})=o(1)$, and the second term is $O(n^{-1})=o(1)$ by applying eq.~(\ref{eq:vz}) with $Z_{ij}:=(\bs x_i,\bs x_j,d_i,d_j),h(Z_{ij})=\mu_{ij}(\bs \theta_*) \log \mu_{ij}(\bs \theta)-\mu_{ij}(\bs \theta)$. 
Therefore, $V_{\bs W,\bs X,\bs d}[\tilde{n}^{-1}\ell_n(\bs \theta)]=o(1)$ and Chebyshev's inequality implies the pointwise convergence 
$\tilde{n}^{-1}\ell_n(\bs \theta) \overset{p}{\to} \ell(\bs \theta)$ for every $\bs \theta \in \bs \Theta$ where 
$\ell(\bs \theta)
=
\ave_{\bs W,\bs X,\bs d}[\tilde{n}^{-1}\ell_n(\bs \theta)]
=
\ave_{\bs x_1,\bs x_2,d_1,d_2}[\mu_{12}(\bs \theta_*) \log \mu_{12}(\bs \theta)-\mu_{12}(\bs \theta)]$. Thus, condition (ii) holds. 
%\item \textbf{Stochastic equicontinuity} of $\tilde{n}^{-1}\ell_n(\bs \theta)$ is derived from its Lipschitz continuity over the compact set $\bs \Theta$. 

%\item[(iii)] 
Finally, we work on condition (iv).
Since $\mu_{ij}(\bs \theta)$ is a composite function of $C^1$-functions on $\bs \Theta$, $\mu_{ij}(\bs \theta)$ is Lipschitz continuous. 
The Lipschitz continuity of $\mu_{ij}(\bs \theta)$ and $\mu_{ij}(\bs \theta)>0 \: (\bs \theta \in \bs \Theta)$ indicates the Lipschitz continuity of $\log \mu_{ij}(\bs \theta)$.  
Therefore, there exist $M_1,M_2>0$ such that
\begin{align*}
|\tilde{n}^{-1}\ell_n(\bs \theta)-\tilde{n}^{-1}\ell_n(\bs \theta')|
&\leq
\biggl|
\frac{1}{\tilde{n}}\sum_{(i,j) \in \mathcal{I}_n}
w_{ij}(\log \mu_{ij}(\bs \theta)-\log \mu_{ij}(\bs \theta'))
-
\frac{1}{\tilde{n}}\sum_{(i,j) \in \mathcal{I}_n}
(\mu_{ij}(\bs \theta)-\mu_{ij}(\bs \theta'))
\biggr| \\
&\leq
\frac{1}{\tilde{n}}\sum_{(i,j) \in \mathcal{I}_n}
w_{ij} |\log \mu_{ij}(\bs \theta)-\log \mu_{ij}(\bs \theta')|
+
\frac{1}{\tilde{n}}\sum_{(i,j) \in \mathcal{I}_n}
|\mu_{ij}(\bs \theta)-\mu_{ij}(\bs \theta')| \\
&\leq
M_1 \biggl( {\tilde{n}}^{-1}\sum_{(i,j) \in \mathcal{I}_n} w_{ij}  \biggr)
 \|\bs \theta-\bs \theta'\|_2 
+ 
M_2\|\bs \theta-\bs \theta'\|_2.
\end{align*}
Denoting by $B_n:=M_1 \cdot {\tilde{n}}^{-1}\sum_{(i,j) \in \mathcal{I}_n} w_{ij} + M_2$, we have
\begin{align*}
|\tilde{n}^{-1}\ell_n(\bs \theta)-\tilde{n}^{-1}\ell_n(\bs \theta')|
&\leq
B_n \|\bs \theta-\bs \theta'\|_2.
\end{align*}
Since ${\tilde{n}}^{-1}\sum_{(i,j) \in \mathcal{I}_n} w_{ij}=O_p(1)$, the law of large numbers indicates $B_n=O_p(1)$. 
Thus, condition (iv) holds. 
%\end{enumerate}
\qed
\end{proof}

%\begin{remark}
%Although we consider only the convergence in probability for the simplicity, Theorem~\ref{theo:uniform_convergence} can be extended by \cite{ha2014remark} to almost uniform convergence without imposing any additional condition. 
%\end{remark}
%Theorem~\ref{theo:uniform_convergence} indicates the following theorem, which describes the consistency of the MLE. 

Noticing that $\bs \theta_*$ is a maximizer of $\ell(\bs \theta)$ and $\hat{\bs \theta}_n$ is a maximizer of $\ell_n(\bs \theta)$, we would have the desired result $\hat{\bs \theta}_n \overset{p}{\to}  \bs \theta_*$ 
by combining Theorem~\ref{theo:uniform_convergence} and continuity of $\ell(\bs \theta)$.
However it does not hold unfortunately.
Instead, we define the set of parameter values equivalent to $\bs\theta_*$ as
$\bs \Theta_*:=\{\bs \theta \in \bs \Theta \mid \ell(\bs \theta)=\ell(\bs \theta_*)\}$.
Every $\bs \theta \in \bs \Theta_*$ gives the correct probability of link weights, because
$\ell(\bs \theta)=\ell(\bs \theta_*)$ holds if and only if $\mu_{12}(\bs \theta)=\mu_{12}(\bs \theta_*)$ almost surely w.r.t. $(\bs x_1,\bs x_2,d_1,d_2)$.
With this setting, the theorem below states that $\hat{\bs \Theta}_n$ converges to $\bs \Theta_*$ in probability.
This indicates that, $\hat{\bs\theta}_n$ will represent the true probability model for sufficiently large $n$.
\begin{theo}
\label{theo:consistency_general}
\normalfont
Let $d_{\text{H}}(\cdot,\cdot)$ denote the Hausdorff distance defined as $\max$-$\min$ $L^2$-distance between two sets.  We assume the same conditions as in Theorem~\ref{theo:uniform_convergence}.
Then we have, as $n\to\infty$,
\begin{align}
	d_H(\hat{\bs \Theta}_n,\bs \Theta_*) \overset{p}{\to} 0.
\label{eq:dh}
\end{align}
\end{theo}

%(Chernozhukov2007で$\hat{c}=c_n=1$とした場合だけ考える) 

\begin{proof}
\normalfont 
We refer to the case (1) of Theorem 3.1 in \citet{chernozhukov2007estimation} with $\hat{c}=1$ under the condition C.1
This theorem shows, for general setting of $\hat {\bs \Theta}_I,\bs \Theta_I$, that
\begin{align}
d_H(\hat{\bs \Theta}_I,\bs \Theta_I) \overset{p}{\to} 0.
\label{eq:chernozhukov}
\end{align}
Here 
 $\hat{\bs \Theta}_I:=\{\bs \theta \in \bs \Theta \mid \hat{Q}_n(\bs \theta) \leq 1/a_n\}$ and
 $\bs \Theta_I:=\arginf_{\bs \theta \in \bs \Theta}Q(\bs \theta)$,
where
$\hat{Q}_n(\bs\theta),Q(\bs \theta)$ are general functions satisfying 
$\sup_{\bs \theta \in \bs \Theta_I}Q_n(\bs \theta)=o_p(1/a_n)$,
and $a_n \to \infty$.

For proving (\ref{eq:dh}), we consider the case of $\hat{Q}_n(\bs \theta)=-\tilde{n}^{-1}\ell_n(\bs \theta)+\sup_{\bs \theta \in \bs \Theta}\tilde{n}^{-1}\ell_n(\bs \theta)$, $Q(\bs \theta)=-\ell(\bs \theta)+\sup_{\bs \theta \in \bs \Theta} \ell(\bs \theta)$. 
% and $a_n=n^{-\varepsilon}$ for a sufficiently small $\varepsilon>0$ such that $\sup_{\bs \theta \in \bs \Theta_I} \hat{Q}_n(\bs \theta) = o_p(1/a_n)$. 
The condition C.1 for (\ref{eq:chernozhukov}) is re-written as follows.
(i) $\bs \Theta$ is a (non-empty) compact set, 
(ii) $\tilde{n}^{-1}\ell_n(\bs \theta)$ and $\ell(\bs \theta)$ are continuous, 
(iii) $\sup_{\bs \theta \in \bs \Theta}|\tilde{n}^{-1}\ell_n(\bs \theta)-\ell(\bs \theta)| \overset{p}{\to} 0$, and 
(iv) $\sup_{\bs \theta \in \bs \Theta_*}(-\tilde{n}^{-1}\ell_n(\bs \theta)+\sup_{\bs \theta \in \bs \Theta}\tilde{n}^{-1}\ell_n(\bs \theta)) \overset{p}{\to} 0$. 
The conditions (i), (ii) are obvious.  (iii) is shown in Theorem~\ref{theo:uniform_convergence}.
(iv) is verified by
$$
	\sup_{\bs \theta \in \bs \Theta_*}
	\left(-\tilde{n}^{-1}\ell_n(\bs \theta) + \sup_{\bs \theta \in \bs \Theta} \tilde{n}^{-1}\ell_n(\bs \theta) \right)
	=
	-\inf_{\bs \theta \in \bs \Theta_I} \tilde{n}^{-1}\ell_n(\bs \theta) + \sup_{\bs \theta \in \bs \Theta} \tilde{n}^{-1}\ell_n(\bs \theta)
	\overset{p}{\to}
	-\ell(\bs \theta_*) + \ell(\bs \theta_*)
	=
	0,
$$
where $\bs \theta_*$ is an element of $\bs \Theta_I$. 
Thus, (\ref{eq:chernozhukov}) holds.

Next, we consider two sets
$\hat{\bs \Theta}_n=\argsup_{\bs \theta \in \bs \Theta} \tilde{n}^{-1}\ell_n(\bs \theta)=\{\bs \theta \in \bs \Theta \mid Q_n(\bs \theta)=0\}$ and $\hat{\bs \Theta}_I$.  Since these sets satisfy $Q_n(\bs \theta)=0 \: (\bs \theta \in \hat{\bs \Theta}_n), Q_n(\bs \theta')\leq 1/a_n \: (\bs \theta' \in \hat{\bs \Theta}_I)$ and $1/a_n \to 0$ as $n \to \infty$, we have $d_H(\hat{\bs \Theta}_n,\hat{\bs \Theta}_I)
\overset{p}{\to}
0$.
It follows from this convergence and (\ref{eq:chernozhukov}) that,
 by noticing $\bs \Theta_*=\bs \Theta_I$, 
 \begin{align*}
d_H(\hat{\bs \Theta}_n,\bs \Theta_*)
=
d_H(\hat{\bs \Theta}_n,\bs \Theta_I)
\leq
d_H(\hat{\bs \Theta}_n,\hat{\bs \Theta}_I)
+
d_H(\hat{\bs \Theta}_I,\bs \Theta_I)
\overset{p}{\to}
0, 
\end{align*}
thus (\ref{eq:dh}) holds. \qed
\end{proof}

\section{CDMCA is approximated by PMvGE with linear transformations}
\label{sec:optimization_linear_transformation}
%\section{PMvGE approximately generalizes CDMCA, which already generalizes various methods for feature learning}
%\label{sec:mle_supp}

We argue an approximate relation between CDMCA and PMvGE, which is briefly explained in Section~\ref{subsec:relation}.
In the below, we will derive the solution $\hat{\bs \psi}_{\text{CDMCA}}$ of a slightly modified version of CDMCA, and an approximate solution $\hat{\bs \psi}_{\text{Apr.PMvGE}}$ of PMvGE with linear transformations.
We then show that these two solutions are equivalent up to a scaling in each axis of the shared space.

%\begin{enumerate}[{(1)}]
%\item %(1)
\subsection{Solution of a modified CDMCA}
The original CDMCA imposes the quadratic constraint (\ref{eq:cdmca_constraint}) for maximizing the objective function (\ref{eq:cdmca_objective}). Here we replace $w_{ij}$ in  (\ref{eq:cdmca_constraint}) with $\delta_{ij}$ so that the constraint becomes
\[
\sum_{i=1}^{n}\bs \psi^{(d_i)\top}\bs x_i\bs x_i^{\top}\bs \psi^{(d_i)}=\bs I.
\]
This modification changes the scaling in the solution, but the computation below is essentially the same as that in \citet{shimodaira2016cross}.
Let us define the augmented data vector, called ``simple coding'' \citep{shimodaira2016cross}, 
$\tilde{\bs x}_i:=(\bs 0_{p_1},\ldots,\bs 0_{p_{d_i-1}},\bs x_i,\bs 0_{p_{d_i+1}},\ldots,\bs 0_{p_D}) \in \mathbb{R}^p$
where $p:=p_1+p_2+\cdots+p_D$. 
Now, data matrix is $\bs X:=(\tilde{\bs x}_1^{\top},\tilde{\bs x}_2^{\top},\ldots,\tilde{\bs x}_n^{\top})^{\top} \in \mathbb{R}^{n \times p}$, and the parameter matrix is $\bs \psi:=(\bs \psi^{(1)\top},\bs \psi^{(2)\top},\ldots,\bs \psi^{(D)\top})^{\top}\in \mathbb{R}^{p \times K}$. With this augmented representation, the $D$-view embedding is now interpreted as a 1-view embedding.
CDMCA maximizes the objective function
\begin{align*}
\frac{1}{2}
\sum_{i=1}^{n}
\sum_{j=1}^{n}
w_{ij} \langle \bs \psi^{(d_i)\top}\bs x_i,\bs \psi^{(d_j)\top}\bs x_j \rangle
=
\text{tr} \left(\bs \psi^{\top}\bs H\bs \psi \right) \quad
(\bs H:=\bs X^{\top}\bs W\bs X),
\end{align*}
with respect to $\bs \psi$ under constraint $\bs \psi^{\top}\bs G\bs \psi=\bs I$ where $\bs G:=\bs X^{\top}\bs X$. 
Let $\bs U_K$ be the matrix composed of the top-$K$ eigenvectors of $\bs G^{-1/2}\bs H\bs G^{-1/2}$.
Then the solution of the modified CDMCA is
\begin{align}
\hat{\bs \psi}_{\text{CDMCA}}:=\bs G^{-1/2}\bs U_K.
\nonumber
%\label{eq:solution_CDMCA_supp}
\end{align}

%\item 
\subsection{Approximate solution of PMvGE with linear transformations} 
MLE of PMvGE maximizes ${\ell}_n(\bs \alpha, \bs \psi) $ defined in (\ref{eq:log_likelihood}).
Here we modify it by adding an extra term as 
$\tilde{\ell}_n(\bs \alpha,\bs \psi):=\ell_n(\bs \alpha,\bs \psi)-\frac{1}{2}\sum_{i:(d_i,d_i) \in \mathcal{D}}\mu_{ii}(\bs \alpha,\bs \psi)$. The difference approaches zero for large $n$, because
$|\ell_n(\bs \alpha,\bs \psi)-\tilde{\ell}_n(\bs \alpha,\bs \psi)|/|\ell_n(\bs \alpha,\bs \psi)| = O(n^{-1})$.
Since the parameter $\bs \alpha$ is not considered in CDMCA, we assume $\mathcal{D}:=\{\text{all pairs of views}\}$ and $\bar{\bs \alpha}=(\bar{\alpha}^{(de)}),\bar{\alpha}^{(de)} \equiv \alpha_0>0 \: (\forall d,e)$. 
We further assume that the transformation of PMvGE is linear: $f^{(d)}_{\bs \psi}(\bs x)=\bs \psi^{(d)\top} \bs x \: (\forall d)$, 
and data vectors in each view are centered. 
With this setting, we will show that the maximizer of a quadratic approximation of $\tilde{\ell}_n$ is equivalent to CDMCA.

To rewrite the likelihood function as $\tilde{\ell}_n(\bar{\bs \alpha},\bs \psi)=\frac{1}{2}\sum_{i=1}^{n}\sum_{j=1}^{n}S_{ij}(g_{ij}(\bs \psi))$, 
we define $g_{ij}(\bs \psi):=\langle \bs \psi^{(d_i)\top}\bs x_i,\bs \psi^{(d_j)\top}\bs x_j \rangle$ and $S_{ij}(g):=w_{ij} \log (\alpha_0 \exp(g) ) - \alpha_0 \exp(g)$, $g\in \mathbb{R}$. 
Since $S_{ij}(g)$ is approximated quadratically around $g=0$ by 
\begin{align*}
S^Q_{ij}(g)
=
\alpha_0
\left\{
-\frac{1}{2}g^2
+
\left( \frac{w_{ij}}{\alpha_0}-1 \right)
g
\right\}
+
S_{ij}(0),
\end{align*}
$\tilde{\ell}_n(\bar{\bs \alpha},\bs \psi)$ is approximated quadratically by 
\begin{align}
\tilde{\ell}^Q_n(\bs \psi)
&:=
\frac{1}{2}
\sum_{i=1}^{n}\sum_{j=1}^{n} S^Q_{ij}(g_{ij}(\bs \psi)) \nonumber \\
&=
\alpha_0
\left\{
	-\frac{1}{2}\tr \left(
	(\bs \psi^{\top} \bs G \bs \psi)^2
	\right)
	+
	\frac{1}{\alpha_0}
	\tr \left( \bs \psi^{\top} \bs H \bs \psi \right)
	-
	\underbrace{
	\tr \left( \bs \psi^{\top} \bs X^{\top}\bs 1_n\bs 1_n^{\top}\bs X \bs \psi \right)
	}_{=0 \: (\because \{\bs x_i\} \text{ is centered.})}
\right\}
+
\underbrace{
\frac{1}{2}
\sum_{i=1}^{n}\sum_{j=1}^{n}
S_{ij}(0)
}_{\text{Const.}}, \label{eq:tilde_ell_n}
\end{align}
%quadratically approximates $\tilde{\ell}_n(\bar{\bs \alpha},\bs \psi)=\frac{1}{2}\sum_{i=1}^{n}\sum_{j=1}^{n}S_{ij}(g_{ij}(\bs \psi))$, 
where $\bs G=\bs X^{\top}\bs X,\bs H=\bs X^{\top}\bs W\bs X$.
The function $\tilde{\ell}^Q_n(\bs \psi)$ has rotational degrees of freedom: $\tilde{\ell}^Q_n(\bs \psi)=\tilde{\ell}^Q_n(\bs \psi \bs O)$ for any orthogonal matrix $\bs O \in \mathbb{R}^{K \times K}$.
Thus, we impose an additional constraint $\bs \psi^{\top}\bs G\bs \psi=\bs \Gamma_K:=\text{diag}(\gamma_1,\gamma_2,\ldots,\gamma_K)$
for any $(\gamma_1,\ldots,\gamma_K) \in \mathbb{R}^K_{\ge0}$.
$\bs \psi$ satisfying this constraint is written as $\bs \psi=\bs G^{-1/2}\bs V_K \bs \Gamma_K^{1/2}$
where $\bs V_K \in \mathbb{R}^{P \times K}$ is a column-orthogonal matrix such that $\bs V_K^{\top}\bs V_K=\bs I$. 
%$\bs \psi^{\top}\bs G\bs \psi=\bs \Gamma_K$より
%$\bs \psi$は$\bs \psi=\bs G^{+1/2}\bs U_K \bs \Gamma_K^{1/2}$ where $\bs U_K \in \mathbb{R}^{P \times K}$ is an orthogonal matrix which $\bs U_K^{\top}\bs U_K$ becomes an identity matrix. これをEq.~(\ref{eq:surrogate})に代入すれば
By substituting $\bs \psi=\bs G^{-1/2}\bs V_K \bs \Gamma_K^{1/2}$ into eq.~(\ref{eq:tilde_ell_n}), we have
\begin{align}
\tilde{\ell}^Q_n(\bs \psi)
=
\alpha_0
\left\{
-\frac{1}{2}\tr \left(\bs \Gamma_K^2 \right)
+
\frac{1}{\alpha_0}
\tr \left( \bs \Gamma_K \bs S_K \right)
\right\}
+
\text{Const.}
=
\frac{\alpha_0}{2}
\left\{
\bigg\|\frac{1}{\alpha_0} \bs S_K\bigg\|_{\text{F}}^2
-
\bigg\|\bs \Gamma_K-\frac{1}{\alpha_0} \bs S_K\bigg\|_{\text{F}}^2
\right\}
+
\text{Const.},
\label{eq:surrogate_2}
\end{align}
where $\bs S_K=\bs V_K^{\top}\bs G^{-1/2}\bs H\bs G^{-1/2}\bs V_K$ and $\|\cdot\|_{\text{F}}$ denotes the Frobenius norm. 
%by $\gamma_1',\gamma_2',\ldots,\gamma_K'$, 
%eq.~(\ref{eq:surrogate_2}) becomes
%\begin{align*}
%\tilde{\ell}^Q_n(\bs \psi)
%=
%-\frac{\alpha_0}{2}
%\left\{
%	\sum_{k=1}^{K} \left(\gamma_k-\frac{1}{\alpha_0}\gamma_k' \right)^2
%\right\}
%+
%\frac{1}{2a_0} \sum_{k=1}^{K} \gamma_k'^2
%+
%\text{Const.}
%\end{align*}
This objective function is maximized when $\bs \Gamma_K=\frac{1}{\alpha_0}\bs S_K$ and $\bs V_K=\bs U_K$,
because 
$\min_{\bs \Gamma_K}\|\bs \Gamma_K-\frac{1}{\alpha_0} \bs S_K\|_{\text{F}}^2=0$ is achieved by
$\bs \Gamma_K = \frac{1}{\alpha_0}\bs S_K$, and
$\max_{\bs V_K}\|\bs S_K\|_{\text{F}}^2$ is achieved by 
$\bs V_K=\bs U_K$ where $\bs U_K$ is the matrix composed of the top-$K$ eigenvectors of $\bs G^{-1/2}\bs H\bs G^{-1/2}$.
Therefore, $\tilde{\ell}^Q_n(\bs \psi)$ is maximized by
\begin{align}
\hat{\bs \psi}_{\text{Apr.PMvGE}}=\bs G^{-1/2}\bs U_K \bs \Gamma_K^{1/2}.
\nonumber
%\label{eq:solution_PMvGE_supp}
\end{align} 
By substituting $\bs V_K=\bs U_K$ into $\bs \Gamma_K=\frac{1}{\alpha_0}\bs S_K$, we verify that
$\bs \Gamma_K$ is a diagonal matrix with $\gamma_k:=\lambda_k/\alpha_0$ where $\lambda_k$ is the $k$-th largest eigenvalue of $\bs G^{-1/2}\bs H\bs G^{-1/2} \: (k=1,2,\ldots,K)$. 
%\end{enumerate}

\subsection{Equivalence of the two solutions up to a scaling} 

By comparing the two solutions, we have
\begin{align}
\hat{\bs \psi}_{\text{Apr.PMvGE}}
=
\hat{\bs \psi}_{\text{CDMCA}}\bs \Gamma_K^{1/2}.
\label{eq:solution_comparison}
\end{align}
This simply means that each axis in the shared space is scaled by the factor $\sqrt{\gamma_k}$, $k=1,\ldots, K$.
Let $\hat{\bs y}$, $\hat{\bs y}'$ be feature vectors in the shared space computed by the approximate PMvGE with linear transformations,
and ${\bs y}$, ${\bs y}'$ be feature vectors in the shared space computed by the modified CDMCA. Then
the inner product is weighted in PMvGE as
\[
	\langle \hat{\bs y}, \hat{\bs y}' \rangle
	 = \sum_{k=1}^K \hat y_k \hat y_k' =  \sum_{k=1}^K \gamma_k y_k  y_k'.
\]

\end{document}